\documentclass[10pt,letterpaper]{article}

\usepackage[top=1in,bottom=1in,left=1in,right=1in]{geometry}

\usepackage{microtype}
\usepackage{graphicx}
\usepackage{booktabs}

\usepackage{hyperref}

\usepackage{url}
\usepackage{amsfonts}
\usepackage{nicefrac}

\usepackage{subcaption}
\usepackage{tabularx}
\usepackage{amsmath}
\usepackage{amssymb}
\usepackage{amsthm}
\usepackage{verbatim}

\DeclareMathOperator*{\argmin}{arg\,min}

\usepackage[labelfont={sl,small}, textfont=small, labelsep=period]{caption} 

\newcommand{\vertiii}[1]{{\left\vert\kern-0.25ex\left\vert\kern-0.25ex\left\vert #1 
		\right\vert\kern-0.25ex\right\vert\kern-0.25ex\right\vert}}

\usepackage{authblk}
\usepackage{appendix}

\usepackage[numbers]{natbib}

\newtheorem{theorem}{Theorem}

\newtheorem{lemma}{Lemma}

\newtheorem{assumption}{Assumption}
\newtheorem{definition}{Definition}

\usepackage{color}

\usepackage[normalem]{ulem}

\title{\textbf{Meta Learning for Support Recovery in High-dimensional Precision Matrix Estimation}}

\author[1]{Qian Zhang}
\author[2]{Yilin Zheng}
\author[3]{Jean Honorio}
\affil[1]{
	Department of Statistics,
	Purdue University,
	\texttt{zhan3761@purdue.edu}
}
\affil[2]{
	Department of Computer Science,
	Purdue University,
	\texttt{zheng453@purdue.edu}
}
\affil[3]{
	Department of Computer Science,
	Purdue University,
	\texttt{jhonorio@purdue.edu}
}

\date{}

\begin{document}

\maketitle

\begin{abstract}
In this paper, we study meta learning for support (i.e., the set of non-zero entries) recovery in high-dimensional precision matrix estimation where we reduce the sufficient sample complexity in a novel task with the information learned from other auxiliary tasks. In our setup, each task has a different random true precision matrix, each with a possibly different support. We assume that the union of the supports of all the true precision matrices (i.e., the true support union) is small in size. We propose to pool all the samples from different tasks, and \emph{improperly} estimate a single precision matrix by minimizing the $\ell_1$-regularized log-determinant Bregman divergence. We show that with high probability, the support of the \emph{improperly} estimated single precision matrix is equal to the true support union, provided a sufficient number of samples per task $n \in O((\log N)/K)$, for $N$-dimensional vectors and $K$ tasks. That is, one requires less samples per task when more tasks are available. We prove a matching information-theoretic lower bound for the necessary number of samples, which is $n \in \Omega((\log N)/K)$, and thus, our algorithm is minimax optimal. Then for the novel task, we prove that the minimization of the $\ell_1$-regularized log-determinant Bregman divergence with the additional constraint that the support is a subset of the estimated support union could reduce the sufficient sample complexity of successful support recovery to $O(\log(|S_{\text{off}}|))$ where $|S_{\text{off}}|$ is the number of off-diagonal elements in the support union and is much less than $N$ for sparse matrices. We also prove a matching information-theoretic lower bound of $\Omega(\log(|S_{\text{off}}|))$ for the necessary number of samples. Synthetic experiments validate our theory.
\end{abstract}

\section{Introduction}
Precision (or inverse covariance) matrix estimation is an important problem in high-dimensional statistical learning \cite{wang2016precision} with great application in time series \cite{chen2013covariance}, principal component analysis \cite{fan2016overview}, probabilistic graphical models \cite{meinshausen2006high}, etc. For example, in Gaussian graphical models where we model the variables in a graph as a zero-mean multivariate Gaussian random vector, the set of off-diagonal non-zero entries of the precision matrix corresponds exactly to the set of edges of the graph \cite{ravikumar2011high}. For this reason, estimating the precision matrix to recover its support set, which is the set of non-zero entries, is the common strategy of structure learning in Gaussian graphical models. An estimate of the precision matrix is called sign-consistent if it has the same support and sign of entries with respect to the true matrix.

However, the learner faces several challenges in precision matrix estimation. The first challenge is the high-dimensionality of the data. The dimension of the data, $N$, could be much higher than the sample size $n$, and thus the empirical sample covariance and its inverse will behave badly \cite{johnstone2001distribution}. Secondly, unlike in Gaussian graphical models, the data may not follow multivariate Gaussian distribution. The third challenge is the heterogeneity of the data. There could be limited samples from the distribution of interest but a large amount of samples from multiple multivariate distributions with different precision matrices. 

For the first two challenges, we assume the precision matrices are sparse and consider a general class of distributions, i.e., multivariate sub-Gaussian distributions later described in Definition \ref{def:mulsubG}. The class of sub-Gaussian variates \cite{buldygin1980sub} includes for instance Gaussian variables, any bounded random variable (e.g. Bernoulli, multinomial, uniform), any random variable with strictly log-concave density, and any finite mixture of sub-Gaussian variables. Then we address the high-dimension challenge by using $\ell_1$-regularized log-determinant Bregman divergence minimization \cite{ravikumar2011high}, which is also the $\ell_1$-regularized maximum likelihood estimator for multivariate Gaussian distributions \cite{yuan2007model}. 

For the challenge of heterogeneity, prior works have considered a multi-task learning problem where the learner treats each different distribution as a task with a related precision matrix and solves each and every task simultaneously.
Suppose there are $K$ tasks and $n$ samples with dimension $N$ per task. When there is only one task ($K=1$), Ravikumar et al. \cite{ravikumar2011high} proved that $n\in O(\log N)$ is sufficient for the sign-consistency of $\ell_1$-regularized log-determinant Bregman divergence minimization with multivariate sub-Gaussian data. When $K>1$, Honorio et al. \cite{honorio2012statistical} proposed the $\ell_{1,p}$-regularized log-determinant Bregman divergence minimization to estimate the precision matrices of all tasks and proved that $n\in O(\log K+\log N)$ is sufficient for the correct support union recovery with high probability. Guo et al. \cite{guo2011joint} introduced a different regularized maximum likelihood estimation to learn all precision matrices and proved $n\in O((N\log N) /K)$ is sufficient for the correct support recovery of the precision matrix in each task with high probability. Ma and Michailidis \cite{ma2016joint} proposed a joint estimation method consisting of a group Lasso regularized neighborhood selection step and a maximum likelihood step. They proved that their method recovers the support of the precision matrix in each task with high probability if $n\in O(K+\log N)$. There are also several algorithms for the multi-task problem but without theoretical guarantees for the consistency of their estimates \cite{mohan2014node, chiquet2011inferring}.

In this paper, we solve the heterogeneity challenge with meta learning where we recover the support of the precision matrix in a novel task with the information learned from other auxiliary tasks. Unlike previous methods, we also use improper estimation in our meta learning method to have better theoretical guarantees for support recovery. Specifically, instead of estimating each and every precision matrix in the auxiliary tasks, we pool all the samples from the auxiliary tasks together to estimate a single ``common precision matrix'' (see Definition \ref{def:msGGM}) in order to recover the ``support union'' (see Definition \ref{def:msGGM}) of the precision matrices in those tasks. Then we estimate the precision matrix of the novel task with the constraint that its support is a subset of the estimated support union and its diagonal entries are equal to the diagonal entries of the estimated common precision matrix. We prove that for the sign-consistency of our estimates, the sufficient and necessary sample size per auxiliary task is $n\in\Theta((\log N)/K)$ which is much better than the results of the aforementioned multi-task learning methods and enables the learner to gather more tasks (instead of more samples per task) to get a more accurate estimate since the sample complexity is inversely proportional to $K$. The sufficient and necessary sample complexity of the novel task is $\Theta(\log(|S_{\text{off}}|))$ where $|S_{\text{off}}|$ is the number of off-diagonal elements in the support union $S$ and $|S_{\text{off}}|\ll N$ for sparse graphs, which is better than the result in \cite{ravikumar2011high}.

Moreover, to the best of our knowledge, we are the first to introduce randomness in the precision matrices of different tasks while previous methods assume the precision matrix in each task to be deterministic. Our theoretical results hold for a wide class of distributions of the precision matrices under some conditions, which broadens the application scenarios of our method. 
The use of improper estimation in our method is innovative for the problem of support recovery of high-dimensional precision matrices. Our work also fills in the blank of the theory and methodology of meta learning in high-dimensional precision matrix estimation.
Generally, meta learning aims to develop learning approaches that could have good performance on an extensive range of learning tasks and generalize to solve new tasks easily and efficiently with only a few training examples \cite{Vanschoren19}. Thus it is also referred to as $\textit{learning to learn}$ \cite{lake2015human}. Current research mainly focuses on designing practical meta learning algorithms, for instance, \cite{koch2015siamese, vinyals2016matching, sung2018learning, santoro2016meta, munkhdalai2017meta, finn2017model}. We believe our work could provide some insights for the theoretical understanding of meta learning.

This paper has the following four contributions. Firstly, we propose a meta learning approach by introducing multiple auxiliary learning tasks for support recovery of high-dimensional precision matrices with improper estimation. Secondly, we add randomness to the precision matrices in different learning tasks, which is a significant innovation compared to previous methods.
Thirdly, we prove that for $N$-dimensional multivariate sub-Gaussian random vectors and $K$ auxiliary tasks with support union $S$, the sufficient sample complexity of our method is $O((\log N)/K)$ per auxiliary task for support union recovery and $O(\log(|S_{\text{off}}|))$ for support recovery of the novel task, which provides the theoretical basis for introducing more tasks for meta learning in support recovery of precision matrices. 
Fourthly, we prove information-theoretic lower bounds for the failure of support union recovery in the auxiliary tasks and the failure of support recovery in the novel task. We show that $\Omega((\log N)/K)$ samples per auxiliary task and $\Omega(\log(|S_{\text{off}}|))$ samples for the novel task are necessary for the recovery success, which proves that our meta learning method is minimax optimal.
Lastly, we conduct synthetic experiments to validate our theory. We calculate the support union recovery rates of our meta learning approach and multi-task learning approaches for different sizes of samples and tasks. For a fixed task size $K$, our approach achieves high support union recovery rates when the sample size per task has the order $O((\log N)/K)$. For a fixed sample size per task, our method performs the best when the task size $K$ is large.

\section{Preliminaries}
This section introduces our mathematical models and the meta learning problem. The important notations used in the paper are illustrated in Table~\ref{Tab:notations}.

{
	\begin{table*}
		\vskip 0.1in
		\caption{Notations used in the paper}
		\label{Tab:notations}
		\vskip 0.05in
		\centering
		\begin{tabularx}{\textwidth}{lX}
			\toprule
			Notation     & Description     \\
			\midrule
			$\text{sign}(x)$ & The sign of $x\in\mathbb{R}$, i.e., $\text{sign}(x)=x/|x|$ if $x\neq0$; $\text{sign}(x)=0$ if $x=0$    \\
			$\|a\|_{\infty}$ & The $\ell_{\infty}$-norm of vector $a\in\mathbb{R}^n$, i.e., $\max_{i=1}^n|a_i|$   \\
			$\|a\|_{1}$ & The $\ell_{1}$-norm of vector $a\in\mathbb{R}^n$, i.e., $\sum_{i=1}^n|a_i|$   \\
			$\|A\|_{\infty}$ & The $\ell_{\infty}$-norm of matrix $A\in\mathbb{R}^{m\times n}$, i.e., $\max_{1\le i\le m,1\le j\le n}|A_{ij}|$   \\
			$\|A\|_{1}$ & The $\ell_{1}$-norm of matrix $A\in\mathbb{R}^{m\times n}$, i.e., $\sum_{1\le i\le m,1\le j\le n}|A_{ij}|$   \\
			$\vertiii{A}_{\infty}$ & The $\ell_{\infty}$-operator-norm of matrix $A\in\mathbb{R}^{m\times n}$, i.e., $\max_{1\le i\le m}\sum_{j=1}^n|A_{ij}|$   \\
			$\lambda_{\min}(A)$ & The minimum eigenvalue of matrix $A\in\mathbb{R}^{m\times m}$   \\
			$\lambda_{\max}(A)$ & The maximum eigenvalue of matrix $A\in\mathbb{R}^{m\times m}$   \\
			$\vertiii{A}_{2}$ & The $\ell_{2}$-operator-norm of matrix $A\in\mathbb{R}^{m\times n}$, i.e., $\sqrt{\lambda_{\max}(A^{\text{T}}A)}$   \\
			$A\succ0$ & The matrix $A$ is symmetric and positive-definite.   \\
			$\text{det}(A)$ & The determinant of matrix $A\in\mathbb{R}^{m\times n}$    \\
			$\text{supp}(A)$ & The support set of matrix $A\in\mathbb{R}^{m\times n}$, i.e., $\left\{(i,j)|A_{ij}\neq0\right\}$    \\
			$\text{diag}(A)$ & The vector consisting of the diagonal entries of matrix $A\in\mathbb{R}^{n\times n}$, i.e., $[A_{11},A_{22},\dots,A_{nn}]^{\text{T}}$    \\
			$|S|$ & The number of elements in the set $S$ \\
			$S_{\text{off}}$ & The set of off-diagonal elements in the set $S$, i.e., $\{(i,j):(i,j)\in S,i\neq j\}$ \\
			$A_S$ & The sub-matrix composed by the entries according to the set $S$ of $A\in\mathbb{R}^{m\times n}$, i.e., $\left(A_{(i,j)}\right)_{(i,j)\in S}$    \\
			$\langle A,B\rangle\in\mathbb{R}$ & The Frobenius inner product of $A,B\in\mathbb{R}^{m \times n}$, i.e., $\sum_{1\le i\le m,1\le j\le n}A_{ij}B_{ij}$ \\
			$A\odot B\in\mathbb{R}^{m \times n}$ & The Hadamard product of $A,B\in\mathbb{R}^{m \times n}$, i.e., $[A \odot B]_{ij}=A_{ij} B_{ij}$ \\
			$A \otimes B\in\mathbb{R}^{mp \times nq}$ & The Kronecker product of $A\in\mathbb{R}^{m \times n}$, $B\in\mathbb{R}^{p \times q}$, i.e., $[A \otimes B]_{(i,j),(k,l)}=[A \otimes B]_{p(i-1)+k,q(j-1)+l} = A_{ij}B_{kl}$ \\
			$[A\otimes B]_{S_1S_2}$ & The sub-matrix composed by the entries according to the set $S_1\times S_2$ of the matrix $A\otimes B\in\mathbb{R}^{mp\times nq}$ for $A\in\mathbb{R}^{m \times n}$, $B\in\mathbb{R}^{p \times q}$, i.e., $\left([A\otimes B]_{(i,j),(k,l)}\right)_{(i,j)\in S_1,(k,l)\in S_2}$    \\
			\bottomrule
		\end{tabularx}
	\vskip -0.1in
	\end{table*}
}

\subsection{Multivariate sub-Gaussian Distributions with Random Precision Matrices}
We first define a general class of multivariate distributions, the multivariate sub-Gaussian distribution. 
\begin{definition} \label{def:mulsubG}
	We say a random vector $X\in\mathbb{R}^{N}$ follows a multivariate sub-Gaussian distribution with precision $\Omega\in\mathbb{R}^{N\times N}$ and parameter $\sigma$ if
	
	(i) $\mathbb{E}\left[X_{t}^{(k)}\right]=0$, $\text{Cov}\left(X\right)=\Sigma=\left(\Omega\right)^{-1}$, and
	
	(ii) $\frac{X_i}{\sqrt{\Sigma_{ii}}}$ is a sub-Gaussian random variable with parameter $\sigma$ for $1\le i\le N$.
\end{definition}
The definition of sub-Gaussian random variable is as follows \cite{buldygin2000metric}:
\begin{definition} \label{def:subGRV}
	A random variable $X\in\mathbb{R}$ is called sub-Gaussian with parameter $\sigma\ge0$ if
	\begin{equation} \label{eq:DefsubGaussian}
	\mathbb{E}\left[e^{\lambda X}\right]\le \exp\left(\frac{\sigma^2\lambda^2}{2}\right),\ \ \forall\ \lambda\in\mathbb{R}
	\end{equation}
\end{definition}
Obviously, Gaussian variables are sub-Gaussian and the Gaussian graphical model is a special case of the multivariate sub-Gaussian distribution.

In this paper, we consider multiple multivariate sub-Gaussian distributions whose precision matrices are randomly generated, which makes our model more reasonable and universal compared to the deterministic setting in all the previous works. Formally, we define the following family of multivariate sub-Gaussian distributions with random precision matrices:
\begin{definition} \label{def:msGGM}
	Let $X_{1}^{(k)}\,X_{2}^{(k)},...,X_{n^{(k)}}^{(k)}\in\mathbb{R}^{N}$ be i.i.d. random vectors for $1\le k\le K$. Let $X_{t,i}^{(k)}$ be the i-th entry of $X_{t}^{(k)}$ for $1\le i\le N$. We say $\left\{X^{(k)}_t\right\}_{1\le t\le n^{(k)},\ 1\le k\le K}$ follows a family of random $N$-dimensional multivariate sub-Gaussian distributions of size $K$ with parameter $\sigma$ if
	
	(i) $\bar{\Omega}^{(k)}=\bar{\Omega}+\Delta^{(k)}$ with $\bar{\Omega}, \Delta^{(k)}\in\mathbb{R}^{N\times N}$, $\bar{\Omega}\succ0$ deterministic, and $\Delta^{(k)},1\le k\le K,$ are i.i.d. random matrices drawn from distribution $P$;
	
	(ii) For some $\gamma>0,c_{\max}\in\left(0, \lambda_{\min}(\bar{\Omega})/2\right]$, we have
	\begin{equation} \label{eq:condtion2}
	\mathbb{P}_{\Delta\sim P}[\bar{\Omega}+\Delta\succ0, \text{supp}(\Delta)\subseteq\text{supp}(\bar{\Omega}),
	\|(\bar{\Omega}+\Delta)^{-1}\|_{\infty}\le \gamma, \vertiii{\Delta}_2\le c_{\max}]=1
	\end{equation}
	and $\beta:=\|(\bar{\Omega})^{-1}-\mathbb{E}_{\Delta\sim P}[(\bar{\Omega}+\Delta)^{-1}]\|_{\infty}<\infty$;
	
	(iii) $\mathbb{E}\left[X_{t}^{(k)}\big|\bar{\Sigma}^{(k)}\right]=0$, $\text{Cov}\left(X_{t}^{(k)}\big|\bar{\Sigma}^{(k)}\right)=\bar{\Sigma}^{(k)}$ for $\bar{\Sigma}^{(k)}:=\left(\bar{\Omega}^{(k)}\right)^{-1}$, $1\le t\le n^{(k)},\ 1\le k\le K$;
	
	(iv) $\left\{X_{t}^{(k)}\right\}_{1\le t\le n^{(k)},\ 1\le k\le K}$ are conditionally independent given $\{\bar{\Omega}^{(k)}\}_{k=1}^K$;
	
	(v) $\frac{X_{t,i}^{(k)}}{\sqrt{\bar{\Sigma}^{(k)}_{ii}}}$ conditioned on $\bar{\Omega}^{(k)}$ is sub-Gaussian with parameter $\sigma$ for $1\le i\le N,\ 1\le t\le n^{(k)},\ 1\le k\le K$.
	
	We refer to $\bar{\Omega}$ as the true common precision matrix and $S:=\text{supp}(\bar{\Omega})$ as the support union of the above family of distributions.
\end{definition}
Notice that we define the support union as $S=\text{supp}(\bar{\Omega})$ instead of $\cup_{k=1}^K\text{supp}(\bar{\Omega}^{(k)})$ which is a random subset of the deterministic set $S$ because we are interested on a novel task where the support of its precision matrix is a subset of the support of $\Omega$, i.e., $S$.

\subsection{Problem Setting}
\label{subsec:prob}
In this paper, we focus on the problem of estimating the support of the precision matrix of a multivariate sub-Gaussian distribution. Following the principles of meta learning, we solve a novel task by first estimating a superset of the support of the precision matrix in the novel task from $K$ auxiliary tasks.

Specifically, suppose there are $n^{(K+1)}$ samples from a multivariate sub-Gaussian distribution with precision matrix $\bar{\Omega}^{(K+1)}$ for the novel task. We introduce $n^{(k)}$ samples for each auxiliary task $k\in\{1,...,K\}$ and assume all samples in the $K$ auxiliary tasks follow a family of random multivariate sub-Gaussian distributions with common precision matrix $\bar{\Omega}$ specified in Definition \ref{def:msGGM}. Our meta learning method aims to recover the support union $S=\text{supp}(\bar{\Omega})$ with the $K$ auxiliary tasks and use $S$ to assist in recovering $S^{(K+1)}:=\text{supp}(\bar{\Omega}^{(K+1)})$ with the assumption that $S^{(K+1)}\subseteq S$.

\section{Our Novel Improper Estimation Method}
As illustrated in Section \ref{subsec:prob}, in the first step of our method, we recover the support union $S$ of the $K$ auxiliary tasks by estimating the true common precision matrix $\bar{\Omega}$. To be specific, we pool all samples from the $K$ tasks together and estimate $\bar{\Omega}$ by minimizing the $\ell_1$-regularized log-determinant Bregman divergence between the estimate and $\bar{\Omega}$; i.e., we solve the following optimization problem with regularization constant $\lambda>0$:
\begin{equation} \label{eq:estimator0}
\hat{\Omega}=\argmin_{\Omega\succ0}\sum_{k=1}^KT^{(k)}\left(\log\det\left(\Omega\right)-
\langle\hat{\Sigma}^{(k)},\Omega\rangle\right)-\lambda\|\Omega\|_1
\end{equation}
where $\hat{\Sigma}^{(k)}:=\frac{1}{n^{(k)}}\sum_{t=1}^{n^{(k)}}X_{t}^{(k)}\left(X_{t}^{(k)}\right)^{\text{T}}$ is the empirical sample covariance and $T^{(k)}$ is proportional to the number of samples $n^{(k)}$ for task $k$. Define the following loss function:
\begin{equation} \label{eq:loss}
\ell(\Omega)=\sum_{k=1}^KT^{(k)}\left(\langle\hat{\Sigma}^{(k)},\Omega\rangle-\log\det\left(\Omega\right)\right)
\end{equation}
Then we can rewrite \eqref{eq:estimator0} as
\begin{equation} \label{eq:estimator}
\hat{\Omega}=\argmin_{\Omega\succ0}\left(\ell(\Omega)+\lambda\|\Omega\|_1\right)
\end{equation}

For clarity of exposition, we assume the number of samples per auxiliary task is the same, i.e., $n^{(k)}=n$, $T^{(k)}=1/K$ for $1\le k\le K$ in our analysis. In addition, we do not assume $n^{(K+1)}=n$. Notice that \eqref{eq:estimator} is an improper estimation because we estimate a single precision matrix with data from different distributions. This will enable us to recover the support union with the most efficient sample size per task (see Section \ref{sec:suppunion}).

For the second step, suppose that we have successfully recovered the true support union $S$ in the first step. Then for a novel task, i.e., the $(K+1)$-th task, since we have assumed the support $S^{(K+1)}$ of its precision matrix $\hat{\Omega}^{(K+1)}$ is also a subset of the support union $S$, we propose the following constrained $\ell_1$-regularized log-determinant Bregman divergence minimization for $\hat{\Omega}^{(K+1)}$:
\begin{equation}\label{eq:estimatorNovel1}
\begin{aligned}
\hat{\Omega}^{(K+1)}=&\argmin_{\Omega\succ0} \ell^{(K+1)}(\Omega)+\lambda\|\Omega\|_1\ \\&\text{s.t.}\ \ \text{supp}(\Omega)\subseteq \text{supp}(\hat{\Omega}),\ \ \text{diag}(\Omega)=\text{diag}(\hat{\Omega}).
\end{aligned}
\end{equation}
where $\ell^{(K+1)}(\Omega):=\langle\hat{\Sigma}^{(K+1)},\Omega\rangle-\log\det\left(\Omega\right)$, $\hat{\Sigma}^{(K+1)}:=\frac{1}{n^{(K+1)}}\sum_{t=1}^{n^{(K+1)}}X_{t}^{(K+1)}\left(X_{t}^{(K+1)}\right)^{\text{T}}$ is the empirical sample covariance and $\hat{\Omega}$ is obtained in \eqref{eq:estimator}. 
Note that \eqref{eq:estimatorNovel1} is also an improper estimation because of the constraint $\text{diag}(\Omega)=\text{diag}(\hat{\Omega})$. For our target of support recovery and sign-consistency, there is no need to estimate the diagonal entries of the precision matrix since they are always positive. Hence, we introduce this constraint to reduce the sample complexity by only focusing on estimating the off-diagonal entries (see Section \ref{sec:suppnovel}).

\section{Theoretical Results}
In this section, we formally state our assumptions and theoretical results.

\subsection{Assumptions}
Our theoretical results require an assumption on the true common precision matrix $\bar{\Omega}$ which is called mutual incoherence or irrepresentability condition in \cite{ravikumar2011high}. The Hessian of the loss function \eqref{eq:loss} when $\Omega=\bar{\Omega}$ is
\begin{equation} \label{eq:D2lossTrue}
\nabla^2\ell\left(\bar{\Omega}\right)=T\bar{\Gamma}
\end{equation}
where $T:=\sum_{k=1}^KT^{(k)}$ and $\bar{\Gamma}:=\nabla^2\log\det(\bar{\Omega})=\bar{\Omega}^{-1}\otimes\bar{\Omega}^{-1}\in \mathbb{R}^{N^2\times N^2}$. The mutual incoherence assumption is as follows:
\begin{assumption} \label{assump1}
	There exists some $\alpha\in(0,1]$ such that $\vertiii{\bar{\Gamma}_{S^cS}(\bar{\Gamma}_{SS})^{-1}}_{\infty}\le 1-\alpha$
\end{assumption}
We should notice that $\vertiii{\bar{\Gamma}_{S^cS}(\bar{\Gamma}_{SS})^{-1}}_{\infty}=\max_{u\in S^c}\|\bar{\Gamma}_{u S}(\bar{\Gamma}_{SS})^{-1}\|_1$. Thus this assumption in fact places restrictions on the influence of non-support terms indexed by $S^c$, on the support-based terms indexed by $S$ \cite{ravikumar2011high}.

We also require the mutual incoherence assumption for the precision matrix $\bar{\Omega}^{(K+1)}$ in the novel task:
\begin{assumption} \label{assump:novel}
	There exists $\alpha^{(K+1)}\in(0,1]$ such that 
	\begin{equation} \label{eq:assump-K+1}
	\vertiii{\bar{\Gamma}^{(K+1)}_{(S^{(K+1)})^cS^{(K+1)}}(\bar{\Gamma}^{(K+1)}_{S^{(K+1)}S^{(K+1)}})^{-1}}_{\infty}\le 1-\alpha^{(K+1)}
	\end{equation}
	where $\bar{\Gamma}^{(K+1)}:=(\bar{\Omega}^{(K+1)})^{-1}\otimes(\bar{\Omega}^{(K+1)})^{-1}$.
\end{assumption}
For $\bar{\Gamma}$ and $\bar{\Gamma}^{(K+1)}$, our analysis keeps explicit track of the quantities $\kappa_{\bar{\Gamma}}:=\vertiii{(\bar{\Gamma}_{SS})^{-1}}_{\infty}$ and $\kappa_{\bar{\Gamma}^{(K+1)}}:=\vertiii{(\bar{\Gamma}^{(K+1)}_{S^{(K+1)}S^{(K+1)}})^{-1}}_{\infty}$.

To relate the two norms $\|\cdot\|_{\infty}$ and $\vertiii{\cdot}_{\infty}$, we define the degree of a matrix as the maximal size of the supports of its row vectors. The degree of $\bar{\Omega}$ is
$d:=\max_{1\le i\le N}\left|\left\{j:1\le j\le N,\bar{\Omega}_{ij}\neq0\right\}\right|$ 
and the degree of $\bar{\Omega}^{(K+1)}$ is
$d^{(K+1)}:=\max_{1\le i\le N}\left|\left\{j:1\le j\le N,\bar{\Omega}^{(K+1)}_{ij}\neq0\right\}\right|$.

We call $\bar{\Sigma}:=\bar{\Omega}^{-1}$ the true common covariance matrix and denote its $\ell_{\infty}$-operator-norm by
$\kappa_{\bar{\Sigma}}:=\vertiii{\bar{\Sigma}}_{\infty}$.
Similarly for the covariance matrix $\bar{\Sigma}^{(K+1)}=(\bar{\Omega}^{(K+1)})^{-1}$ in the novel task, we define $\kappa_{\bar{\Sigma}^{(K+1)}}:=\vertiii{\bar{\Sigma}^{(K+1)}}_{\infty}$.

In order to bound $\vertiii{\bar{\Sigma}}_2$ in our proof, we define $\lambda_{\min}:=\lambda_{\min}(\bar{\Omega})$.

To show the sign-consistency of our estimators, we also need to consider the minimal magnitude of non-zero entries in $\bar{\Omega}$ and $\bar{\Omega}^{(K+1)}$, i.e.,
$\omega_{\min}:=\min_{(i,j)\in S}|\bar{\Omega}_{ij}|$,
$\omega_{\min}^{(K+1)}:=\min_{(i,j)\in S}|\bar{\Omega}^{(K+1)}_{ij}|$,

\subsection{Main theorems}
For our meta learning method, we have
\begin{lemma} \label{lem:convex}
	For $\lambda>0$, the problem in \eqref{eq:estimator} and \eqref{eq:estimatorNovel1} are strictly convex and have unique solutions $\hat{\Omega}$ and $\hat{\Omega}^{(K+1)}$ respectively.
\end{lemma}
The detailed proofs of all the lemmas, theorems and corollaries in the paper are in the supplementary material. We then study the theoretical behaviors of $\hat{\Omega}$ in \eqref{eq:estimator} and $\hat{\Omega}^{(K+1)}$ in \eqref{eq:estimatorNovel1}.

\subsubsection{SUPPORT UNION RECOVERY} \label{sec:suppunion}
Our first theorem specifies a probability lower bound of recovering a subset of the true support union by our estimator in \eqref{eq:estimator} for multiple random multivariate sub-Gaussian distributions.
\begin{theorem} \label{thm:subset}
	For a family of $N$-dimensional random multivariate sub-Gaussian distributions of size $K$ with parameter $\sigma$ described in Definition \ref{def:msGGM} with $n^{(k)}=n$, $1\le k\le K$ and satisfying Assumption \ref{assump1}, consider the estimator $\hat{\Omega}$ obtained in \eqref{eq:estimator} with $T^{(k)}=1/K$ and
	$\lambda=(8\delta+4\delta^*)/\alpha$ for $\delta\in(0,\delta^*/2]$ where $\delta^*:=\frac{\alpha^2}{2\kappa_{\bar{\Gamma}}(\alpha+8)^2}\min\left\{\frac{1}{3\kappa_{\bar{\Sigma}}d},\frac{1}{3\kappa_{\bar{\Sigma}}^3\kappa_{\bar{\Gamma}}d}\right\}$. If $\beta\le\delta^*/2$, then with probability at least
	\begin{equation} \label{eq:probsubsetsub}
	\begin{aligned}
	1-2N(N+1)\exp\left(-\frac{nK}{2}\min\left\{\frac{\delta^2}{64(1+4\sigma^2)^2\gamma^2},1\right\}\right)
	-2N\exp\left(-\frac{K\lambda_{\min}^4}{128c_{\max}^2}\left(\frac{\delta^*}{2}-\beta\right)^2\right)
	\end{aligned}
	\end{equation}
	we have:
	
	(i) $\text{supp}(\hat{\Omega})\subseteq\text{supp}\left(\bar{\Omega}\right)$
	
	(ii) $\|\hat{\Omega}-\bar{\Omega}\|_{\infty}\le \kappa_{\bar{\Gamma}}\left(\frac{8}{\alpha}+1\right)(2\delta+\delta^*)$
\end{theorem}
\begin{proof}[Proof sketch for Theorem \ref{thm:subset}]
	We use the primal-dual witness approach \cite{ravikumar2011high} to prove Theorem \ref{thm:subset}. The key step is to verify that the strict dual feasibility condition holds. Using some norm inequalities and Brouwer's fixed point theorem (see e.g. \cite{ortega2000iterative}), we show that it suffices to bound the random term $\|\sum_{k=1}^K\frac{1}{K}W^{(k)}\|_{\infty}$ with $W^{(k)}=\hat{\Sigma}^{(k)}-\bar{\Sigma}$ for $1\le k\le K$ after some careful and involved derivation. Then we decompose the random term into two parts as follows
	\begin{equation*}
	\begin{aligned}
	\|\sum_{k=1}^K\frac{1}{K}W^{(k)}\|_{\infty}=&\|\frac{1}{K}\sum_{k=1}^K\hat{\Sigma}^{(k)}-
	\bar{\Sigma}^{(k)}+\bar{\Sigma}^{(k)}-\bar{\Sigma}\|_{\infty} \\ \le &
	\underbrace{\|\frac{1}{K}\sum_{k=1}^K\hat{\Sigma}^{(k)}-\bar{\Sigma}^{(k)}\|_{\infty}}_{Y_1}+
	\underbrace{\|\frac{1}{K}\sum_{k=1}^{K}\bar{\Sigma}^{(k)}-\bar{\Sigma}\|_{\infty}}_{Y_2}
	\end{aligned}
	\end{equation*}
	Conditioning on $\{\bar{\Sigma}^{(k)}\}_{k=1}^K$, $Y_1$ can be bounded by the sub-Gaussianity of the samples. Then by the law of total expectation we can get the term $2N(N+1)\exp\left(-\frac{nK}{2}\min\left\{\frac{\delta^2}{64(1+4\sigma^2)^2\gamma^2},1\right\}\right)$ in \eqref{eq:probsubsetsub}.
	
	Define $H:=\frac{1}{K}\sum_{k=1}^{K}\bar{\Sigma}^{(k)}$. We bound $Y_2$ with the following two terms
	\begin{equation} \label{eq:boundY2}
	\begin{aligned}
	Y_2&=\|H-\mathbb{E}[H]+\mathbb{E}[H]-\bar{\Sigma}\|_{\infty}\\&\le
	\|\mathbb{E}[H]-\bar{\Sigma}\|_{\infty}+\vertiii{H-\mathbb{E}[H]}_{2}\\&=
	\beta+\vertiii{H-\mathbb{E}[H]}_{2}
	\end{aligned}
	\end{equation}
	since $\|\mathbb{E}[H]-\bar{\Sigma}\|_{\infty}=\|(\bar{\Omega})^{-1}-\mathbb{E}_{\Delta\sim P}[(\bar{\Omega}+\Delta)^{-1}]\|_{\infty}=\beta$. Then we bound $\vertiii{H-\mathbb{E}[H]}_{2}$ with Corollary 7.5 in \cite{Tropp2011} to get the term $2N\exp\left(-\frac{K\lambda_{\min}^4}{128c_{\max}^2}\left(\frac{\delta^*}{2}-\beta\right)^2\right)$ in \eqref{eq:probsubsetsub}. The detailed proof is in the supplementary material.
\qedhere
\end{proof}

Our proof follows the primal-dual witness approach \cite{ravikumar2011high}. From Theorem \ref{thm:subset}, we can see that for our method, a sample complexity of $O((\log N)/K)$ per task is sufficient for the recovery of a subset of the true support union. 

The next theorem addresses the sign-consistency of the estimate \eqref{eq:estimator}. We say the estimator $\hat{\Omega}$ is sign-consistent if 
\begin{equation} \label{eq:consis}
\text{sign}(\hat{\Omega}_{ij})=\text{sign}(\bar{\Omega}_{ij})\ \ \text{for}\ \ \forall i,j\in\{1,2,...,N\}
\end{equation}
It is obvious that sign-consistency immediately implies the success of support recovery.
\begin{theorem} \label{thm:consis}
	For a family of $N$-dimensional random multivariate sub-Gaussian distributions of size $K$ with parameter $\sigma$ described in Definition \ref{def:msGGM} with $n^{(k)}=n$, $1\le k\le K$ and satisfying Assumption \ref{assump1}, consider the estimator $\hat{\Omega}$ obtained in \eqref{eq:estimator} with $T^{(k)}=1/K$ and $\lambda=8\delta^{\dagger}/\alpha$ where 
		\begin{equation*}
		\delta^{\dagger}:=\left\{
		\begin{aligned}
		&\frac{\alpha^2}{2\kappa_{\bar{\Gamma}}(\alpha+8)^2}\min\left\{\frac{1}{3\kappa_{\bar{\Sigma}}d},\frac{1}
		{3\kappa_{\bar{\Sigma}}^3\kappa_{\bar{\Gamma}}d}\right\},\ \ \text{if}\ \omega_{\min}\ge\frac{2\alpha}{8+\alpha}\min\left\{\frac{1}{3\kappa_{\bar{\Sigma}}d},
		\frac{1}{3\kappa_{\bar{\Sigma}}^3\kappa_{\bar{\Gamma}}d}\right\}; \\
		&\frac{\alpha\omega_{\min}}{4(8+\alpha)\kappa_{\bar{\Gamma}}},\ \ 
		\text{otherwise,}
		\end{aligned}
		\right.
		\end{equation*}
		If $\beta\le\delta^{\dagger}/2$, then with probability at least
		\begin{equation} \label{eq:probconsissub}
		\begin{aligned}
		1-2N(N+1)\exp\left(-\frac{nK}{2}\min\left\{\frac{(\delta^{\dagger})^2}{256(1+4\sigma^2)^2\gamma^2},1\right\}\right)
		-2N\exp\left(-\frac{K\lambda_{\min}^4}{128c_{\max}^2}\left(\frac{\delta^{\dagger}}{2}-\beta\right)^2\right)
		\end{aligned}
		\end{equation}
	the estimator $\hat{\Omega}$ is sign-consistent and thus $\text{supp}(\hat{\Omega})=\text{supp}\left(\bar{\Omega}\right)$.
\end{theorem}
According to Theorem \ref{thm:consis}, a sample complexity of $O((\log N)/K)$ per task is sufficient for the recovery of the true support union by our estimator in \eqref{eq:estimator}.

We also prove the following information-theoretic lower bound on the failure of support union recovery for some family of random multivariate sub-Gaussian distributions.
\begin{theorem} \label{thm:fano}
	For some family of $N$-dimensional random multivariate sub-Gaussian distributions of size $K$ with parameter $\sigma$ and covariance matrices $\{\bar{\Sigma}^{(k)}\}_{k=1}^K$, suppose $N\ge 5$, $\bar{\Sigma}^{(k)}=(I + H \odot Q^{(k)})^{-1}$ for $1\le k\le K$ with $Q^{(k)}\in [-1/(2d), 1/(2d)]^{N\times N}$ symmetric, degree $d\in\mathbb{Z}^{+}$ even and $H\in \{0,1\}^{N\times N}$ such that $H$ is symmetric and $H_{ij}=1$ iff $(i,j)\in E$. Thus $S:=E\cup\{(i,i)\}_{i=1}^N$ is the support union of all precision matrices. Assume $E$ is randomly generated in the following way:
	
	(i) Obtain a permutation $\pi=(\pi_1,\pi_2,...,\pi_N)$ of $V=\{1,2,...,N\}$ uniformly at random.
	
	(ii) Let $\pi_{N+j}:=\pi_j$ for $1\le j\le d/2$
	
	(iii) For $i=1,...,N$, add $(\pi_i,\pi_{i+j})$ to $E$ for $1\le j\le d/2$.
	
	Thus $d$ is the degree of the precision matrices in all tasks. Suppose that for each of the $K$ distributions, we have $n$ samples randomly drawn from them. Then for any estimate $\hat{S}$ of $S$, we have
	\begin{equation} \label{eq:probfano}
	\mathbb{P}\{\hat{S}\neq S\} \geq 1 - \frac{nNK + \log 2}{N\log N - N - \log{2N}}
	\end{equation}
\end{theorem}
\begin{proof}[Proof sketch for Theorem \ref{thm:fano}]
	For the random set $S$, random samples $\mathbf{X}=\{X_t^{(k)}\}_{1\le t\le n,1\le k\le K}$, and $\mathbf{Q}:=\{Q^{(k)}\}_{k=1}^K$, we prove that the conditional entropy $H(S|\textbf{Q}) = \log( (N-1)!/2)$ and the conditional mutual information $\mathbb{I}(\mathbf{X};S|\mathbf{Q})\le nNK$.
	
	By the Fano’s inequality extension in \cite{ghoshal2017information}, we have
	\begin{equation*}
	\begin{aligned}
	\mathbb{P}\{\hat{S}\neq S\} \ge
	1 - \frac{\mathbb{I}(\textbf{X};S|\textbf{Q})+\log 2}{H(S|\textbf{Q})}
	\ge 1-\frac{nNK + \log 2}{\log [(N-1)!/2]}
	\end{aligned}
	\end{equation*}
	which leads to \eqref{eq:probfano}. The detailed proof is in the supplementary material.
\end{proof}
According to Theorem \ref{thm:fano}, if the sample size per distribution is $n\le(\log N)/(2K)-1/(2K)-(\log(8N))/(2NK)$, then with probability larger than $1/2$, any method will fail to recover the support union of the multiple random multivariate sub-Gaussian distributions specified in Theorem \ref{thm:fano}. Thus a sample complexity of $\Omega((\log N)/K)$ per task is necessary for the support union recovery of the $N$-dimensional multivariate sub-Gaussian distributions in $K$ tasks, which, combined with Theorem \ref{thm:consis}, indicates that our estimate \eqref{eq:estimator} is minimax optimal with a necessary and sufficient sample complexity of $\Theta((\log N)/K)$ per task.

\subsubsection{SUPPORT RECOVERY FOR NOVEL TASK} \label{sec:suppnovel}
For the novel task, the next theorem proves a probability lower bound for the sign-consistency of the estimate \eqref{eq:estimatorNovel1}.
\begin{theorem} \label{thm:novel}
	Suppose we have recovered the true support union $S$ of a family of $N$-dimensional random multivariate sub-Gaussian distributions of size $K$ with parameter $\sigma$ described in Definition \ref{def:msGGM} with $n^{(k)}=n$ for $k=1,..,K$. For a novel task of multivariate sub-Gaussian distribution with precision matrix $\bar{\Omega}^{(K+1)}$ such that $\text{supp}(\bar{\Omega}^{(K+1)})\subseteq S$ and satisfying Assumption~\ref{assump:novel}, consider the estimator $\hat{\Omega}^{(K+1)}$ obtained in \eqref{eq:estimatorNovel1} with 
		$\lambda=\frac{8\delta^{(K+1),\dagger}}{\alpha^{(K+1)}}$ where 
		\begin{equation*}
		\delta^{(K+1),\dagger}:=\left\{
		\begin{aligned}
		&\frac{(\alpha^{(K+1)})^2}{2\kappa_{\bar{\Gamma}}(\alpha^{(K+1)}+8)^2d^{(K+1)}}\min\Big\{\frac{1}{3\kappa_{\bar{\Sigma}^{(K+1)}}},
		\frac{1}{3\kappa_{\bar{\Sigma}^{(K+1)}}^3\kappa_{\bar{\Gamma}^{(K+1)}}}\Big\},\\ 
		&\quad\quad\quad \text{if}\ \omega_{\min}^{(K+1)}\ge\frac{2\alpha^{(K+1)}}{(8+\alpha^{(K+1)})d^{(K+1)}}
		\min\left\{\frac{1}{3\kappa_{\bar{\Sigma}^{(K+1)}}},
		\frac{1}{3\kappa_{\bar{\Sigma}^{(K+1)}}^3\kappa_{\bar{\Gamma}^{(K+1)}}}\right\}, \\
		&\frac{\alpha^{(K+1)}\omega^{(K+1)}_{\min}}{4(8+\alpha^{(K+1)})\kappa_{\bar{\Gamma}^{(K+1)}}},\ \ 
		\text{otherwise.}
		\end{aligned}
		\right.
		\end{equation*}
		If $\|\bar{\Sigma}^{(K+1)}\|_{\infty}\le \gamma^{(K+1)}$, then with probability at least,
		\begin{equation} \label{eq:newprobconsissub2}
		\begin{aligned}
		1 - 2|S_{\text{off}}|\exp\left(-\frac{n^{(K+1)}}{2}
		\min\left\{\frac{(\delta^{(K+1),\dagger})^2}
		{64(1+4\sigma^2)^2(\gamma^{(K+1)})^2},1\right\}\right)
		\end{aligned}
		\end{equation}
	the estimator $\hat{\Omega}^{(K+1)}$ is sign-consistent and thus $\text{supp}(\hat{\Omega}^{(K+1)})=\text{supp}\left(\bar{\Omega}^{(K+1)}\right)$.
\end{theorem}
\begin{proof}[Proof sketch for Theorem \ref{thm:novel}]
	We use the primal-dual witness approach. Since we have two constraints in \eqref{eq:estimatorNovel1}, we can consider the Lagrangian
	\begin{equation} \label{eq:lagrangian_ske}
	L(\Omega,\mu,\nu) =\ell^{(K+1)}(\Omega)+\lambda\|\Omega\|_1+\langle\mu,\Omega\rangle+\langle\nu, \text{diag}(\Omega-\hat{\Omega})\rangle
	\end{equation}
	where $\mu\in\mathbb{R}^{N\times N}, \nu\in\mathbb{R}^{N}$ are the Lagrange multipliers satisfying $\mu_S=0$. Here we set $\mu=(\bar{\Sigma}^{(K+1)}_{S^c},0)$ (i.e., entries of $\mu$ with index in $S$ equal 0 and entries of $\mu$ with index in $S^c$ equal corresponding entries of $\bar{\Sigma}$) and $\nu=\text{diag}(\bar{\Sigma}^{(K+1)}-\hat{\Sigma}^{(K+1)})$ in \eqref{eq:lagrangian_ske}. Then we can show that it suffices to bound $W^{(K+1)}:=[\hat{\Sigma}^{(K+1)}-\bar{\Sigma}^{(K+1)}]_{S_{\text{off}}}$ for the strict dual feasibility condition to hold. $W^{(K+1)}$ can be bounded by the sub-Gaussianity of the data. The detailed proof is in the supplementary material.
\qedhere
\end{proof}
This theorem shows that $n^{(K+1)}\in O(\log(|S_{\text{off}}|))$ is sufficient for recovering the true support of the novel task with our estimate \eqref{eq:estimatorNovel1}. Therefore, the overall sufficient sample complexity for the sign-consistency of the estimators in the two steps of our meta learning approach is $O(\log(N)/K)$ for each auxiliary task and $O(\log(|S_{\text{off}}|))$ for the novel task.

We also prove the following information-theoretic lower bound for the failure of support recovery for some random multivariate sub-Gaussian distribution where the support set is a subset of a known set $S_{\text{off}}$.
\begin{theorem} \label{thm:fano_novel}
	For $n$ samples generated from some $N$-dimensional multivariate sub-Gaussian distribution with $N\ge 4$,
	suppose the true covariance matrix is $\bar{\Sigma}=(I + H \odot Q)^{-1}$ with $Q\in [-\frac{1}{N\log s},\frac{1}{N\log s}]^{N\times N}$ symmetric and $H\in \{0,1\}^{N\times N}$ such that $H$ is symmetric and $H_{ij}=1$ iff $(i,j)\in E^{(K+1)}$. Thus $S^{(K+1)}:=E^{(K+1)}\cup\{(i,i)\}_{i=1}^N$ is the support set of the precision matrix of this distribution. Assume $E^{(K+1)}$ is chosen uniformly at random from the edge set family $\mathcal{E}:=\{E\subseteq S_{\text{off}}:(i,j)\in E\implies(j,i)\in E\}$ for a known edge set $S_{\text{off}}$. Define $s:=|S_{\text{off}}|$. Assume $4\le s\le N$. 
	Then for any estimate $\hat{S}^{(K+1)}$ of $S^{(K+1)}$, we have
	\begin{equation} \label{eq:probfano_novel}
	\mathbb{P}\{\hat{S}^{(K+1)}\neq S^{(K+1)}\} \geq 1-\frac{4n}{(\log2)(\log s)}-\frac{2}{s}
	\end{equation}
\end{theorem}
\begin{proof}[Proof sketch for Theorem \ref{thm:fano_novel}]
	For the random set $S^{(K+1)}$, random vectors $\mathbf{X}=\{X_t\}_{t=1}^n$, and $Q$, we prove that the conditional entropy $H(S^{(K+1)}|Q) = \log|\mathcal{E}|\ge \frac{s}{2}\log2$ and the conditional mutual information $\mathbb{I}(\mathbf{X};S^{(K+1)}|Q)\le \frac{2ns}{\log s}$.
	
	By the Fano’s inequality extension in \cite{ghoshal2017information}, we have
	\begin{equation*}
	\begin{aligned}
	\mathbb{P}\{\hat{S}^{(K+1)}\neq S^{(K+1)}\} & \ge
	1 - \frac{\mathbb{I}(\textbf{X};S^{(K+1)}|Q)+\log 2}{H(S^{(K+1)}|Q)}\\ &
	\ge 1-\frac{4n}{(\log2)(\log s)}-\frac{2}{s}
	\end{aligned}
	\end{equation*}
	The detailed proof is in the supplementary material.
\end{proof}
According to Theorem~\ref{thm:fano_novel}, if $n\le \frac{\log2}{8}\log s-\frac{\log2}{2s}\log s$, then $\mathbb{P}\{S^{(K+1)}\neq \hat{S}^{(K+1)}\} \ge\frac{1}{2}$, which indicates that the necessary sample complexity for the support recovery of the novel task is $\Omega(\log s)=\Omega(\log |S_{\text{off}}|)$ and our estimate \eqref{eq:estimatorNovel1} is minimax optimal. Therefore, our two-step meta learning method is minimax optimal.

\subsection{Computational Complexity}
Several algorithms have been developed to solve the $\ell_1$-regularized log-determinant Bregman divergence minimization \cite{hsieh2012divide, hsieh2013big, johnson2012high, cai2011constrained}. We have proved in Lemma \ref{lem:convex} that the problems in \eqref{eq:estimator} and \eqref{eq:estimatorNovel1} are convex, which therefore can be solved in polynomial time with respect to the dimension of the random vector $N$ by using interior point methods \cite{boyd2004convex}. Further, state-of-the-art methods for inverse covariance estimation can potentially scale to a million variables \cite{hsieh2013big}.

\section{Validation Experiments}
We validate our theories with synthetic experiments by reporting the success rate for the recovery of the support union. We simulate Erdos-Renyi random graphs in this experiment and compare the results of our estimator in \eqref{eq:estimator} with four multi-task learning methods

For Figure \ref{fig:bysamp}, we fix the number of auxiliary tasks $K=10$ and run experiments with sample size per auxiliary task $n=(C\log N)/K$ for $C$ ranging from 5 to 200. We can see that our method sucessfully recovers the true support union with probability close to 1 when the sample size per auxiliary task is in the order of $O((\log N)/K)$ while the four multi-task learning methods fail. This result provides experimental evidence for Theorem \ref{thm:consis}.

For Figure \ref{fig:bytasks}, we run experiments for different number of auxiliary tasks $K$ that ranges from 2 to 100 with the sample size per auxiliary task $n=200(\log N)/K$. According to Figure \ref{fig:bytasks}, for our method, the support union recovery probability increases with $K$ and converges to 1 for $K$ large enough. For the four multi-task learning methods, however, the probability decreases to 0 as $K$ grows. The results indicate that even with a small number of samples per auxiliary task, we can get a sufficiently accurate estimate using our meta learning method by introducing more auxiliary tasks. 

The details of the simulation and other real-world data experiments are in the supplementary material.

\begin{figure}[ht]
	\vskip 0.1in
	\centering
	\begin{subfigure}{0.32\columnwidth}
		\centering
		\includegraphics[width=\linewidth]{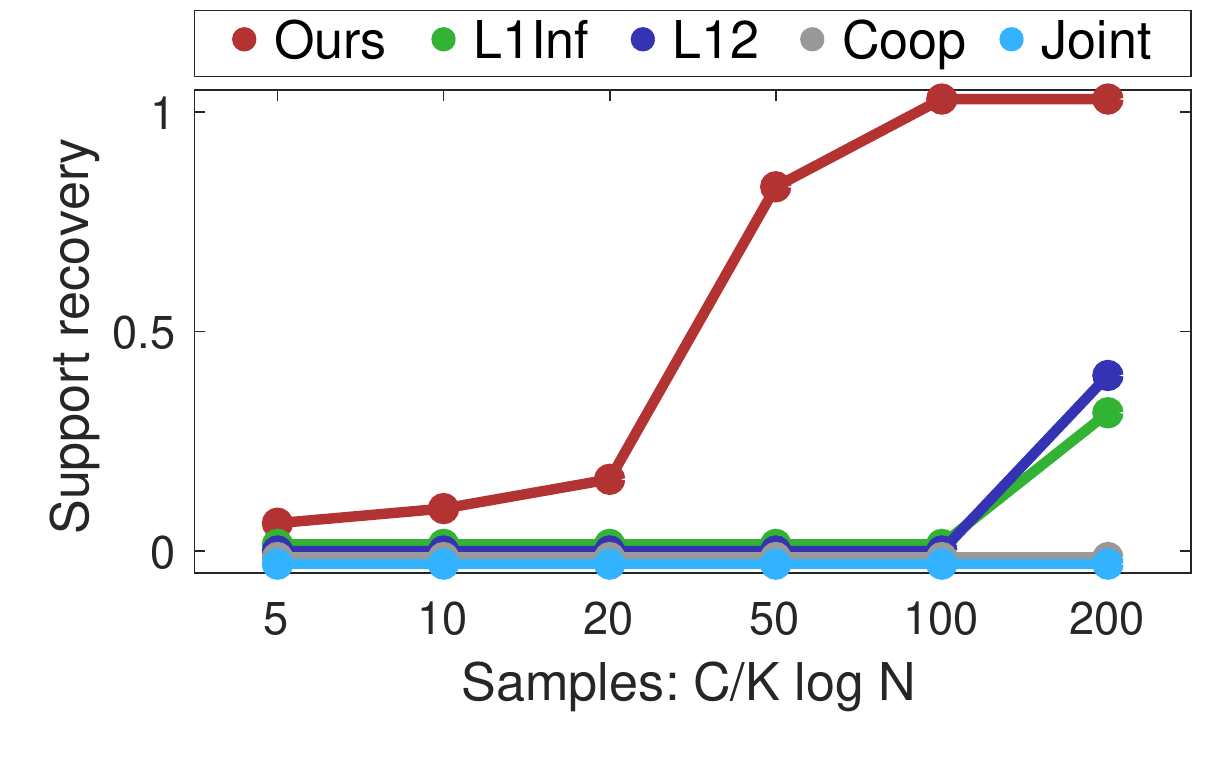}
		\centering
		\caption{Results when $N=10$}
		\label{fig:bysamp10}
	\end{subfigure}
	\begin{subfigure}{0.32\columnwidth}
		\centering
		\includegraphics[width=\linewidth]{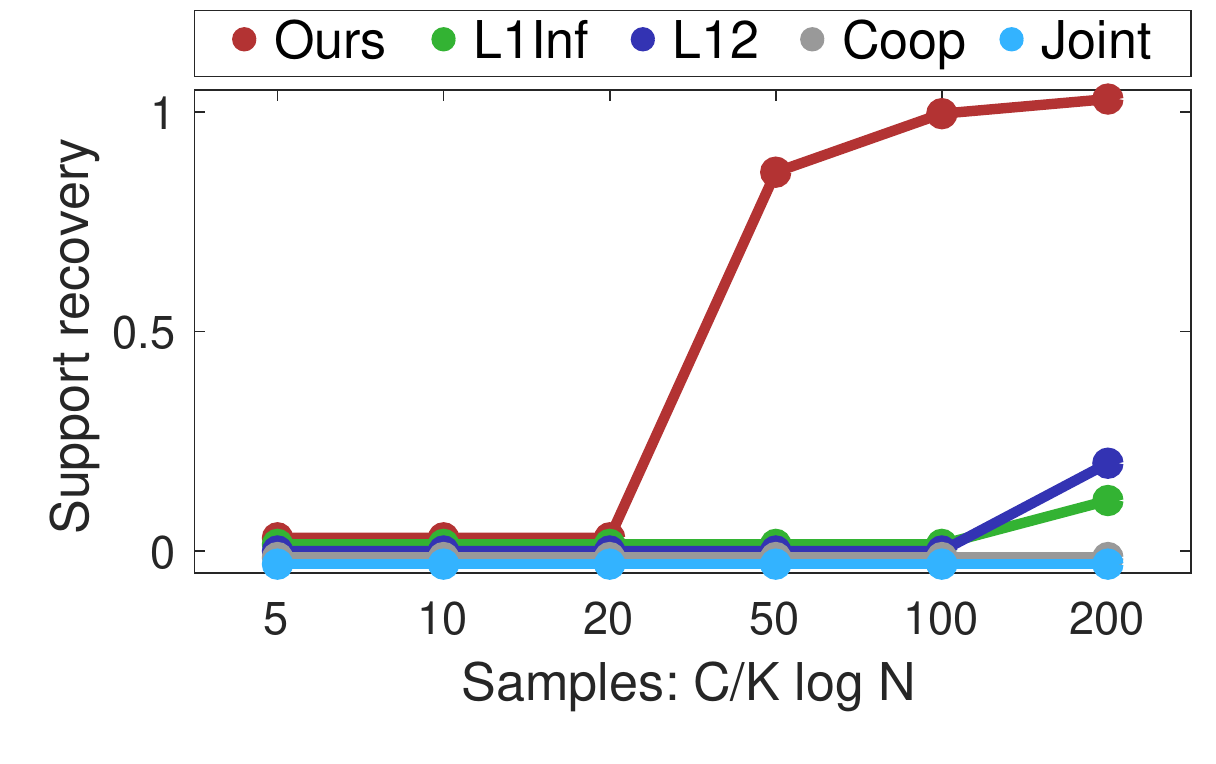}
		\centering
		\caption{Results when $N=20$}
		\label{fig:bysamp20}
	\end{subfigure}
	\begin{subfigure}{0.32\columnwidth}
		\vskip 0.1in
		\centering
		\includegraphics[width=\linewidth]{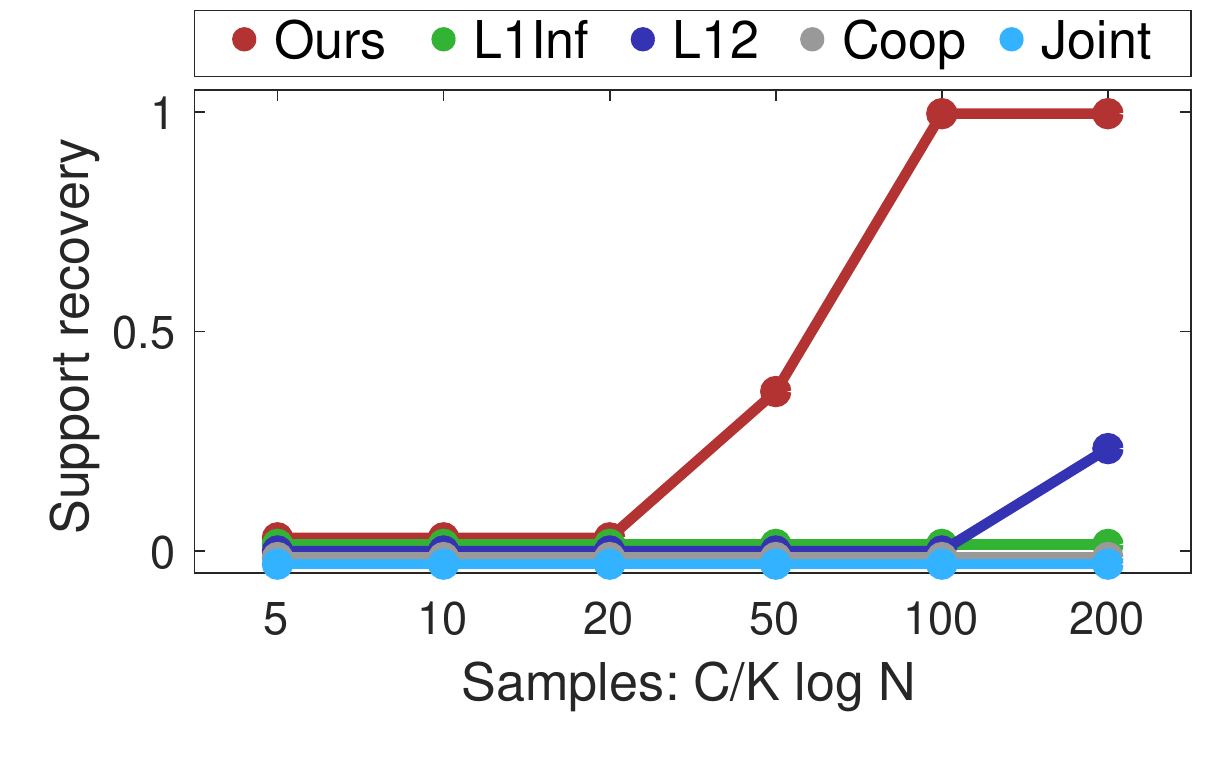}
		\centering
		\caption{Results when $N=50$}
		\label{fig:bysamp50}
	\end{subfigure}
	\caption{The success rate of support union recovery for different sample size $n=(C\log N)/K$ and task size $K=10$. Y-axis shows the success probability and X-axis shows the values of $C$. ``Ours'' is our meta learning method, which we compare against several multitask methods. ``L1Inf'' is the $\ell_{1,\infty}$-regularized method \cite{Honorio10b}. ``L12'' is the $\ell_{1,2}$-regularized method \cite{Varoquaux10}. ``Coop'' is the Cooperative-LASSO method in \cite{chiquet2011inferring}. ``Joint'' is the joint estimation method in \cite{guo2011joint}.}
	\label{fig:bysamp}
	\vskip -0.2in
\end{figure}

\begin{figure}[ht]
	\vskip 0.1in
	\centering
	\begin{subfigure}{0.32\columnwidth}
		\centering
		\includegraphics[width=\linewidth]{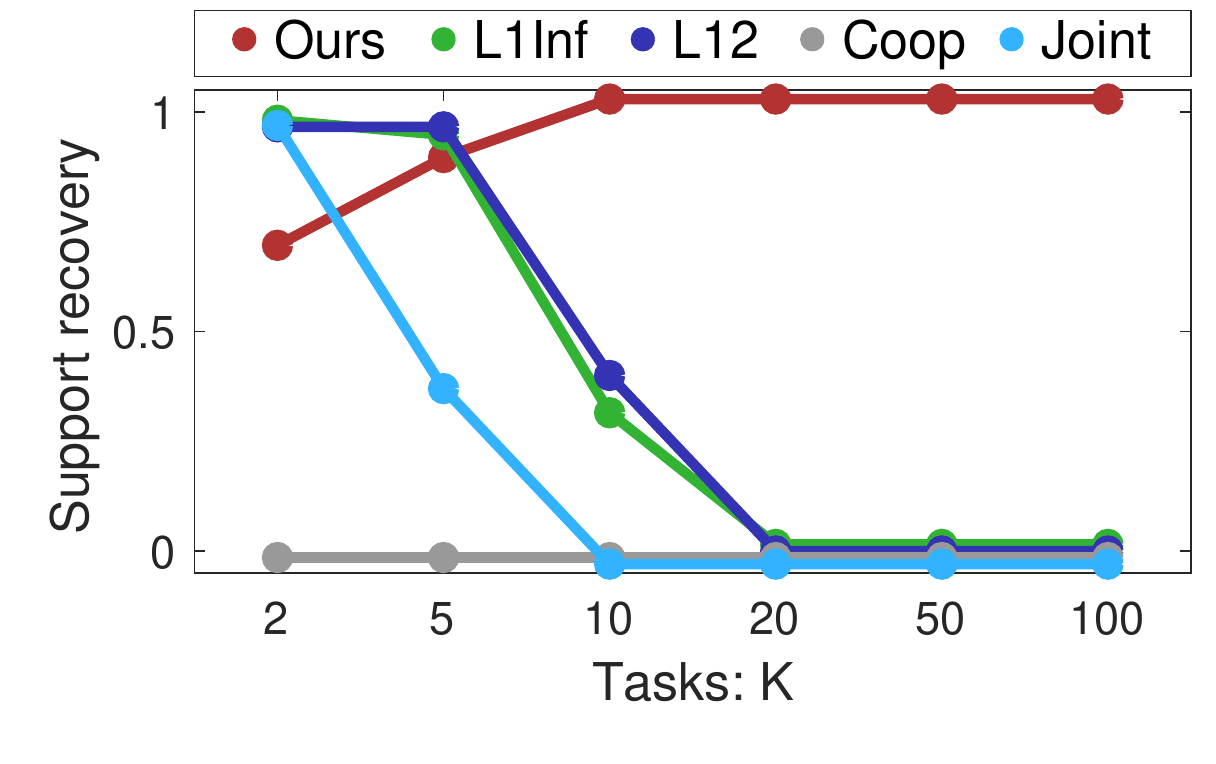}
		\centering
		\caption{Results when $N=10$}
		\label{fig:bytasks10}
	\end{subfigure}
	\begin{subfigure}{0.32\columnwidth}
		\centering
		\includegraphics[width=\linewidth]{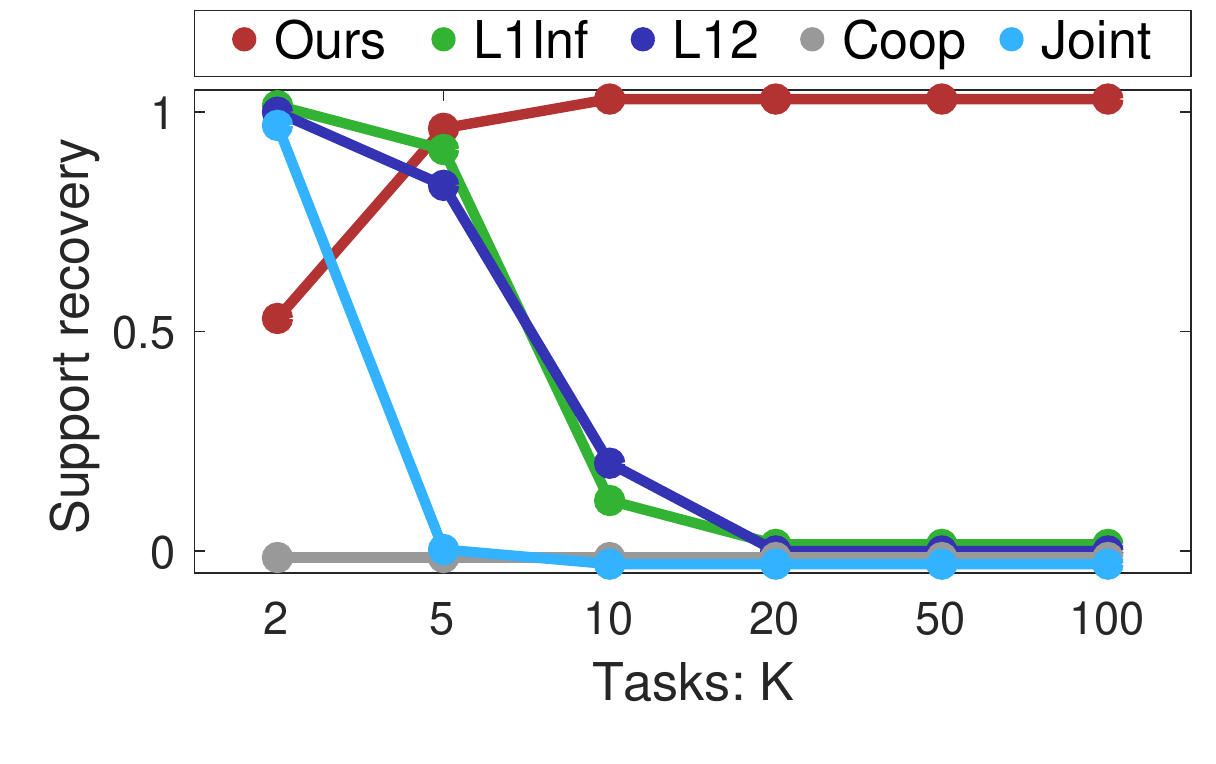}
		\centering
		\caption{Results when $N=20$}
		\label{fig:bytasks20}
	\end{subfigure}
	\begin{subfigure}{0.32\columnwidth}
		\vskip 0.1in
		\centering
		\includegraphics[width=\linewidth]{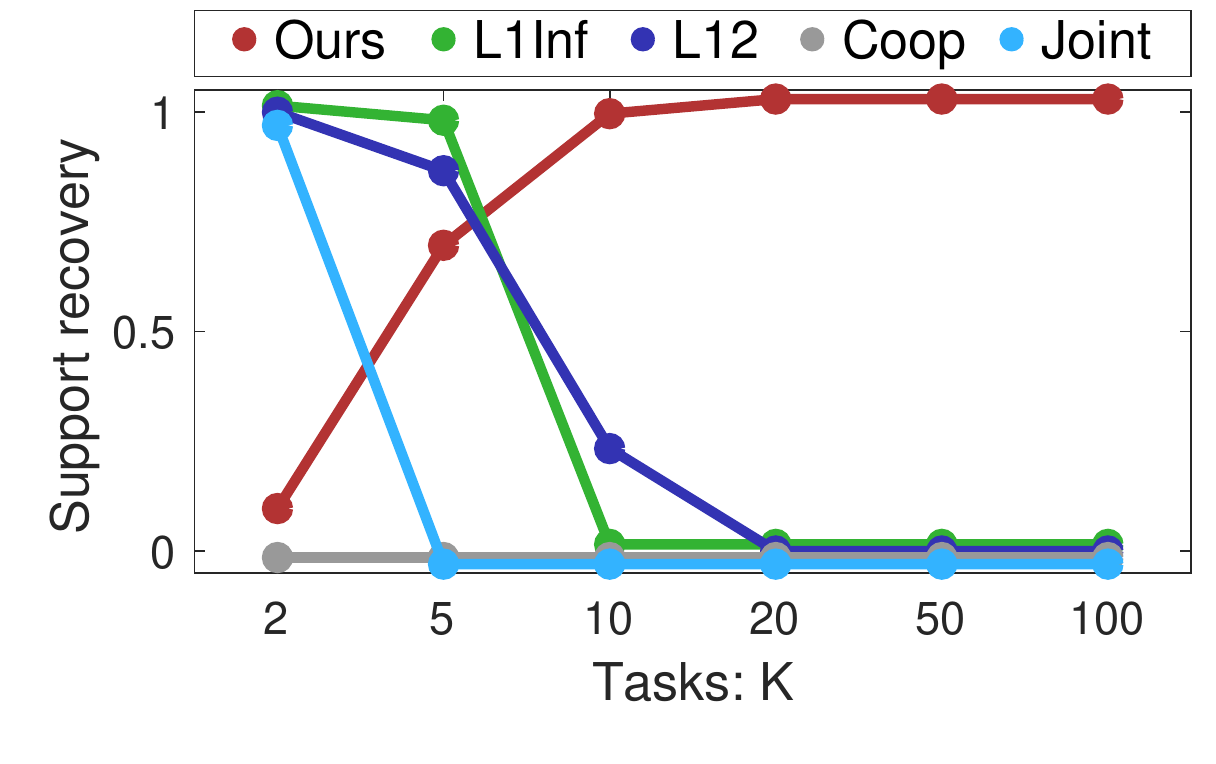}
		\centering
		\caption{Results when $N=50$}
		\label{fig:bytasks50}
	\end{subfigure}
	\caption{The success rate of support union recovery for different task size $K$ with the sample size per task $n=(200\log N)/K$. Y-axis shows the success probability and X-axis shows the value of $K$. ``Ours'' is our meta learning method, which we compare against several multitask methods. ``L1Inf'' is the $\ell_{1,\infty}$-regularized method \cite{Honorio10b}. ``L12'' is the $\ell_{1,2}$-regularized method \cite{Varoquaux10}. ``Coop'' is the Cooperative-LASSO method in \cite{chiquet2011inferring}. ``Joint'' is the joint estimation method in \cite{guo2011joint}.}
	\label{fig:bytasks}
	\vskip -0.2in
\end{figure}

\section{Conclusion}
We develop a meta learning approach for support recovery in precision matrix estimation. Specifically, we pool all the samples from $K$ auxiliary tasks with $K$ random precision matrices, and estimate a single precision matrix by $\ell_{1}$-regularized log-determinant Bregman divergence minimization to recover the support union of the auxiliary tasks. Then we estimate the precision matrix of the novel task with the constraint that its support set is a subset of the support union to reduce the sufficient sample complexity. 
We prove that the sample complexities of $O((\log N)/K)$ per auxiliary task and $O(\log(|S_{\text{off}}|))$ for the novel task are sufficient for our estimators to recover the support union and the support of the precision matrix of the novel task. We also prove that our meta learning method is minimax optimal. Synthetic experiments are conducted and validate our theoretical results.

\clearpage

\bibliographystyle{plainnat}
\bibliography{references}

\clearpage

\begin{appendices}

\section{Additional Experiments}
\subsection{Details of Validation Experiments}
In our simulation experiment in Section 5, we use the Erdos-Renyi random graphs. We first generate $\bar{\Omega}$ by assigning an edge with probability $d/(N-1)$ for each pair of nodes $(i,j)$. Then for each edge $(i,j)$, we set $\bar{\Omega}_{ij}=\bar{\Omega}_{ji}$ to $1$ with probability 0.5 and to $-1$ otherwise. For $1\le k\le K$ and $(i,j)\in S$, $\bar{\Omega}^{(k)}_{ij}$ is set to $\bar{\Omega}_{ij}X_{ij}$ with $X_{ij}\sim$Bernoulli$(0.9)$. Then we add some constants to the diagonal elements of all the precision matrices to ensure they are positive-definite. 

\subsection{Real-world Data Experiment}
We conducted an experiment with real-world data from \cite{kouno2013temporal} using our two-step meta learning method. 
The dataset contains 8 tasks, each with 120 samples. We use 10 samples of each task 1 to 7 to recover the support union and then use 10 samples of task 8 to recover its precision matrix. 
In Table \ref{Tab:realdata}, we report the negative log-determinant Bregman divergence (i.e., the log-likelihood of a multivariate Gaussian distribution) of our meta-learning method for task 8 and compare it with the results of multi-task methods.
\begin{table}[h!]
	\vskip 0.1in
	\caption{Negative log-determinant Bregman divergence of the estimated precision matrices of task 8 using different methods.}
	\label{Tab:realdata}
	\vskip 0.05in
	\centering
	\begin{tabularx}{\textwidth}{lX}
		\toprule
		Method     & Negative log-determinant Bregman divergence     \\
		\midrule
		Our meta learning method & -47    \\
		The $\ell_{1,\infty}$-regularized method \cite{Honorio10b} & -179   \\
		The $\ell_{1,2}$-regularized method \cite{Varoquaux10} & -100   \\
		The Cooperative-LASSO method in \cite{chiquet2011inferring} & -85   \\
		The joint estimation method in \cite{guo2011joint} & -534   \\
		The graphical lasso method (applied only on task 8) in \cite{friedman2008sparse}  & -324   \\
		\bottomrule
	\end{tabularx}
	\vskip -0.1in
\end{table}

According to Table \ref{Tab:realdata}, our method generalizes the best since it obtains the minimum log-determinant Bregman divergence.

\section{Proof of Lemma \ref{lem:convex}}
Define $\mathcal{S}_{++}^{N}:=\{A\in\mathbb{R}^{N\times N}|A\succ0\}$. We first prove the following result:
\begin{lemma} \label{lem:lossstrictconv}
	For $\ell(\Omega)$ defined in \eqref{eq:loss}, if $\Omega\in\mathcal{S}_{++}^{N}$, then $\ell(\Omega)$ is strictly convex.
\end{lemma}
\begin{proof}
	The gradient of $\ell(\Omega)$ is:
	\begin{equation} \label{eq:gradLoss}
	\nabla\ell\left(\Omega\right)=\sum_{k=1}^KT^{(k)}\left(\hat{\Sigma}^{(k)}-\Omega^{-1}\right)
	\end{equation}
	The Hessian of $\ell(\Omega)$ is:
	$$
	\nabla^2\ell\left(\Omega\right)=T\Gamma\left(\Omega\right)
	$$
	where $\Gamma\left(\Omega\right)=\Omega^{-1}\otimes\Omega^{-1}\in \mathbb{R}^{N^2\times N^2}$.
	
	Since $\Omega\in\mathcal{S}_{++}^{N}$, we have $\Omega\succ0$ and thus $\Omega^{-1}\succ0$. 
	According to Theorem 4.2.12 in \cite{horn1994topics}, any eigenvalue of $\Gamma\left(\Omega\right)=\Omega^{-1}\otimes\Omega^{-1}$ is the product of two eigenvalues of $\Omega^{-1}$, hence positive. Therefore,
	$$
	\Gamma\left(\Omega\right)\succ0
	$$
	$$
	\nabla^2\ell\left(\Omega\right)\succ0
	$$
	$\ell(\Omega)$ is strictly convex.
\end{proof}
Now consider $\ell(\Omega)+\lambda\|\Omega\|_1$. Since $\lambda>0$, by Lemma \ref{lem:lossstrictconv}, we know $\ell(\Omega)+\lambda\|\Omega\|_1$ is strictly convex for $\Omega\in\mathcal{S}_{++}^{N}$. Therefore, the problem in \eqref{eq:estimator} is strict convex and has a unique solution $\hat{\Omega}$.

For $\hat{\Omega}^{(K+1)}$ in \eqref{eq:estimatorNovel1}, we have
\begin{equation*}
\nabla\ell^{(K+1)}(\Omega)=\hat{\Sigma}^{(K+1)}-\Omega^{-1}
\end{equation*}
and
\begin{equation*}
\nabla^2\ell^{(K+1)}(\Omega)=\Gamma(\Omega)=\Omega^{-1}\otimes\Omega^{-1}
\end{equation*}
Thus according to the proof of Lemma \ref{lem:lossstrictconv}, we know $\ell^{(K+1)}(\Omega)$ is strictly convex. Then $\ell^{(K+1)}(\Omega)+\lambda\|\Omega\|_1$ is strictly convex for $\lambda>0$ on $\mathcal{S}_{++}^{N}$. Notice that the constraints $\text{supp}(\Omega)\subseteq \text{supp}(\hat{\Omega})$ and $\text{diag}(\Omega)=\text{diag}(\hat{\Omega})$ in \eqref{eq:estimatorNovel1} can be expressed as $\Omega_{ij} = 0$ for $(i,j) \notin S$ and $\Omega_{ii} = \hat{\Omega}_{ii}$ for $i \in \{1,\dots,n\}$. Therefore the constraints are linear. Furthermore, \eqref{eq:estimatorNovel1} is strictly convex for $\lambda>0$ on $\mathcal{S}_{++}^{N}$.

\section{Proof of Theorem \ref{thm:subset}}
Our proof follows the primal-dual witness approach \cite{ravikumar2011high} which uses Karush-Kuhn Tucker conditions (from optimization) together with concentration inequalities (from statistical learning theory). 
\subsection{Preliminaries}
Before the formal proof, we first introduce two inequalities with respect to the matrix $\ell_{\infty}$-operator-norm  $\vertiii{\cdot}_{\infty}$:
\begin{lemma} \label{lem:normIneq}
	For a pair of matrices $A\in\mathbb{R}^{m\times n}$, $B\in\mathbb{R}^{n\times p}$ and a vector $x\in\mathbb{R}^{n}$, we have:
	\begin{equation} \label{eq:normIneq1}
	\|Ax\|_{\infty}\le\vertiii{A}_{\infty}\|x\|_{\infty}
	\end{equation}
	\begin{equation} \label{eq:normIneq2}
	\vertiii{AB}_{\infty}\le\vertiii{A}_{\infty}\vertiii{B}_{\infty}
	\end{equation}
\end{lemma}
\begin{proof}
	Note that
	\begin{align*}
	\|Ax\|_{\infty}&=\max_{1\le i\le m}|\langle a_{i},x\rangle|\\ &\le\max_{1\le i\le m}\|a_{i}\|_1\|x\|_{\infty}\\&
	=\vertiii{A}_{\infty}\|x\|_{\infty}
	\end{align*}
	where $a_i$ is the vector corresponding to the $i$-th row of $A$ and $\langle \cdot,\cdot\rangle$ is the inner product. Similarly, we have
	\begin{align*}
	\|AB\|_{\infty}&=\max_{1\le i\le m}\| a_{i}B\|_1\\ &=\max_{1\le i\le m}\sum_{k=1}^q\Big|\sum_{j=1}^nA_{ij}B_{jk}\Big|\\&
	\le \max_{1\le i\le m}\sum_{j=1}^n|A_{ij}|\sum_{k=1}^q|B_{jk}|\\ & \le
	\max_{1\le i\le m}\sum_{j=1}^n|A_{ij}|\max_{1\le l\le n}\sum_{k=1}^q|B_{lk}| \\ & =
	\max_{1\le i\le m}\sum_{j=1}^n|A_{ij}|\vertiii{B}_{\infty} \\ & =
	\vertiii{A}_{\infty}\vertiii{B}_{\infty}
	\end{align*}
\end{proof}
Then we prove Theorem \ref{thm:subset} with the five steps in the primal-dual witness approach.

\subsection{Step 1} \label{subsection:step1}
Let $\left(\Omega_{S},0\right)$ denote the $N\times N$ matrix such that $\Omega_{S^c}=0$.
For any $\Omega=\left(\Omega_{S},0\right)\in\mathcal{S}_{++}^{N}$, we need to verify that  
$\left[\nabla^2\ell\left(\left(\Omega_{S},0\right)\right)\right]_{SS}\succ0$. 

According to Lemma \ref{lem:lossstrictconv}, since $\left(\Omega_{S},0\right)\in\mathcal{S}_{++}^{N}$, we have
\begin{equation} \label{eq:D2Ssucc}
\nabla^2\ell\left(\left(\Omega_{S},0\right)\right)\succ 0
\end{equation}
Denote the vectorization of a matrix $A$ with $\text{vec}(A)$ or $\overrightarrow{A}$. We use $|S|$ to denote the number of elements in $S$. Then we have $\left[\nabla^2\ell\left(\left(\Omega_{S},0\right)\right)\right]_{SS}\in\mathbb{R}^{|S|\times|S|}$. For $\forall x\in\mathbb{R}^{|S|}$, $v\neq0$, there exists a matrix $A\in\mathbb{R}^{N\times N}$, $A\neq0$, such that $\overrightarrow{A_{S}}=x$. Thus we have
\begin{align*}
x^{\text{T}}\left[\nabla^2\ell\left(\left(\Omega_{S},0\right)\right)\right]_{SS}x&=
\left[\overrightarrow{A_S}\right]^{\text{T}}\left[\nabla^2\ell\left(\left(\Omega_{S},0\right)\right)\right]_{SS}\overrightarrow{A_S}
\\&=\left[\overrightarrow{(A_S,0)}\right]^{\text{T}}\nabla^2\ell\left(\left(\Omega_{S},0\right)\right)\overrightarrow{(A_S,0)}
\\&>0
\end{align*}
where the inequality follows from \eqref{eq:D2Ssucc}. Hence $\left[\nabla^2\ell\left(\left(\Omega_{S},0\right)\right)\right]_{SS}\succ0$. 
Thus the step 1 in primal-dual witness is verified.

\subsection{Step 2}
Construct the primal variable $\tilde{\Omega}$ by making $\tilde{\Omega}_{S^c}=0$ and solving the restricted problem:
\begin{equation} \label{eq:step2}
\tilde{\Omega}_{S}=\arg\min_{\left(\Omega_{S},0\right)\in\mathcal{S}_{++}^N}\ell\left(\left(\Omega_{S},0\right)\right)+\lambda\|\Omega_{S}\|_1
\end{equation}

\subsection{Step 3}
Choose the dual variable $\tilde{Z}$ in order to fulfill the complementary slackness condition of \eqref{eq:estimator}:
\begin{equation} \label{eq:step3}
\left\{
\begin{aligned}
&\tilde{Z}_{ij}=1,\ \ \text{if } \tilde{\Omega}_{ij}>0\\
&\tilde{Z}_{ij}=-1,\ \ \text{if } \tilde{\Omega}_{ij}<0\\
&\tilde{Z}_{ij}\in[-1,1],\ \ \text{if } \tilde{\Omega}_{ij}=0
\end{aligned}
\right.
\end{equation}
Therefore we have
\begin{equation} \label{eq:step3ineq}
\begin{aligned}
\|\tilde{Z}\|_{\infty}\le 1
\end{aligned}
\end{equation}

\subsection{Step 4}
$\tilde{Z}$ is the subgradient of $\|\tilde{\Omega}\|_1$. Solve for the dual variable $\tilde{Z}_{S^c}$ in order that $(\tilde{\Omega},\tilde{Z})$ fulfills the stationarity condition of \eqref{eq:estimator}:
\begin{equation}
\left[\nabla\ell\left(\left(\tilde{\Omega}_S,0\right)\right)\right]_S+\lambda\tilde{Z}_{S}=0
\end{equation}
\begin{equation} \label{eq:station2}
\left[\nabla\ell\left(\left(\tilde{\Omega}_S,0\right)\right)\right]_{S^c}+\lambda\tilde{Z}_{S^c}=0
\end{equation}

\subsection{Step 5}
Now we need to verify that the dual variable solved by Step 4 satisfied the strict dual feasibility condition:
\begin{equation}
\|\tilde{Z}_{S^c}\|_{\infty}<1
\end{equation}
which, according to the stationarity condition, is equivalent to
\begin{equation}
\frac{1}{\lambda}\|\left[\nabla\ell\left(\left(\tilde{\Omega}_{S},0\right)\right)\right]_{S^c}\|_{\infty}<1
\end{equation}
This is the crucial part in the primal-dual witness approach. If we can show the strict dual feasibility condition holds, we can claim that the solution in \eqref{eq:step2} is equal to the solution in \eqref{eq:estimator}, i.e., $\tilde{\Omega}=\hat{\Omega}$. Thus we will have
\begin{equation*}
\text{supp}\left(\hat{\Omega}\right)=\text{supp}\left(\tilde{\Omega}\right)\subseteq S=\text{supp}\left(\bar{\Omega}\right)
\end{equation*}

\subsection{Proof of the strict dual feasibility condition}
Plug the gradient of loss function \eqref{eq:gradLoss} in the stationarity condition of  \eqref{eq:estimator}, we have
\begin{equation} \label{eq:station.1}
\sum_{k=1}^KT^{(k)}\left(\hat{\Sigma}^{(k)}-\tilde{\Omega}^{-1}\right)+\lambda\tilde{Z}=0
\end{equation}
Define $\bar{\Sigma}=\bar{\Omega}^{-1}$, $W^{(k)}:=\hat{\Sigma}^{(k)}-\bar{\Sigma}$, $\Psi:=\tilde{\Omega}-\bar{\Omega}$, 
$R(\Psi):=\tilde{\Omega}^{-1}-\bar{\Sigma}+\bar{\Omega}^{-1}\Psi\bar{\Omega}^{-1}$.
Then we can rewrite \eqref{eq:station.1} as
\begin{equation} \label{eq:station.2}
\sum_kT^{(k)}W^{(k)}+T\left(\bar{\Omega}^{-1}\Psi\bar{\Omega}^{-1}-R(\Psi)\right)+\lambda\tilde{Z}=0
\end{equation}  
From vectorization of product of matrices, we have:
\begin{equation} \label{eq:vecGamma}
\overrightarrow{\bar{\Omega}^{-1}\Psi\bar{\Omega}^{-1}}=\bar{\Gamma}\overrightarrow{\Psi}
\end{equation}
where $\bar{\Gamma}:=\bar{\Omega}^{-1}\otimes\bar{\Omega}^{-1}$. Then vectorize both sides of \eqref{eq:station.2} and we can get:
\begin{equation} \label{eq:station.21}
T\left(\bar{\Gamma}_{SS}\overrightarrow{\Psi_{S}}-\overrightarrow{R_{S}}\right)+\sum_{k=1}^KT^{(k)}\overrightarrow{W^{(k)}_{S}}
+\lambda\overrightarrow{\tilde{Z}_{S}}=0
\end{equation}
\begin{equation} \label{eq:station.22}
T\left(\bar{\Gamma}_{S^cS}\overrightarrow{\Psi_{S}}-\overrightarrow{R_{S^c}}\right)
+\sum_{k=1}^KT^{(k)}\overrightarrow{W^{(k)}_{S^c}}+\lambda\overrightarrow{\tilde{Z}_{S^c}}=0
\end{equation}
where we write $R(\Psi)$ as $R$ for simplicity. By solving \eqref{eq:station.21} for $\overrightarrow{\Psi_{S}}$, we get:
\begin{equation} \label{eq:vecDeltaS}
\overrightarrow{\Psi_{S}}=\frac{1}{T}\bar{\Gamma}_{SS}^{-1}\left(T\overrightarrow{R_S}
-\sum_{k=1}^KT^{(k)}\overrightarrow{W_{S}^{(k)}}-\lambda\overrightarrow{\tilde{Z}_S}\right)
\end{equation}
where we write $(\bar{\Gamma}_{SS})^{-1}$ as $\bar{\Gamma}_{SS}^{-1}$ for simplicity.
Plug \eqref{eq:vecDeltaS} in \eqref{eq:station.22} to solve for $\overrightarrow{\tilde{Z}_{S^c}}$:
\begin{equation*}
\begin{aligned}
\overrightarrow{\tilde{Z}_{S^c}} & =-\frac{1}{\lambda}T\bar{\Gamma}_{S^cS}\overrightarrow{\Psi_{S}}
+\frac{1}{\lambda}T\overrightarrow{R_{S^c}}-\frac{1}{\lambda}\sum_{k=1}^KT^{(k)}\overrightarrow{W^{(k)}_{S^c}}\\ &
=-\frac{1}{\lambda}\bar{\Gamma}_{S^cS}\bar{\Gamma}_{SS}^{-1}\left(T\overrightarrow{R_S}
-\sum_{k=1}^KT^{(k)}\overrightarrow{W_{S}^{(k)}}-\lambda\overrightarrow{\tilde{Z}_S}\right)
+ \frac{1}{\lambda}T\overrightarrow{R_{S^c}}-\frac{1}{\lambda}\sum_{k=1}^KT^{(k)}\overrightarrow{W^{(k)}_{S^c}}\\ &
=-\frac{1}{\lambda}\bar{\Gamma}_{S^cS}\bar{\Gamma}_{SS}^{-1}\left(T\overrightarrow{R_S}
-\sum_{k=1}^KT^{(k)}\overrightarrow{W_{S}^{(k)}}\right)
+\bar{\Gamma}_{S^cS}\bar{\Gamma}_{SS}^{-1}\overrightarrow{\tilde{Z}_S}
+\frac{1}{\lambda}\left(T\overrightarrow{R_{S^c}}-\sum_{k=1}^KT^{(k)}\overrightarrow{W^{(k)}_{S^c}}\right)
\end{aligned}
\end{equation*}
According to \eqref{eq:normIneq1} and the expression above, we have:
\begin{equation*}
\begin{aligned}
\|\overrightarrow{\tilde{Z}_{S^c}}\|_{\infty}\le &\frac{1}{\lambda}
\|\bar{\Gamma}_{S^cS}\bar{\Gamma}_{SS}^{-1}\left(T\overrightarrow{R_S}
-\sum_{k=1}^KT^{(k)}\overrightarrow{W_{S}^{(k)}}\right)\|_{\infty}
+\|\bar{\Gamma}_{S^cS}\bar{\Gamma}_{SS}^{-1}\overrightarrow{\tilde{Z}_S}\|_{\infty}\\ &
+\frac{1}{\lambda}\left(T\|\overrightarrow{R_{S^c}}\|_{\infty}+\|\sum_{k=1}^KT^{(k)}\overrightarrow{W_{S}^{(k)}}\|_{\infty}\right)
\\ \le &\frac{1}{\lambda}\vertiii{\bar{\Gamma}_{S^cS}\bar{\Gamma}_{SS}^{-1}}_{\infty}
\left(T\|\overrightarrow{R_S}\|_{\infty}+\|\sum_{k=1}^KT^{(k)}\overrightarrow{W_{S}^{(k)}}\|_{\infty}\right)\\ &
+\vertiii{\bar{\Gamma}_{S^cS}\bar{\Gamma}_{SS}^{-1}}_{\infty}+\frac{1}{\lambda}
\left(T\|\overrightarrow{R_{S^c}}\|_{\infty}+\|\sum_{k=1}^KT^{(k)}\overrightarrow{W_{S}^{(k)}}\|_{\infty}\right)
\end{aligned}
\end{equation*}
where we have used $\|\overrightarrow{\tilde{Z}_{S}}\|_{\infty}\le1$ by \eqref{eq:step3ineq}.

Therefore under Assumption \ref{assump1}, we have:
\begin{equation*}
\|\tilde{Z}_{S^c}\|_{\infty}=\|\overrightarrow{\tilde{Z}_{S^c}}\|_{\infty}\le \frac{2-\alpha}{\lambda}
\left(T\|\overrightarrow{R}\|_{\infty}+\|\sum_{k=1}^KT^{(k)}\overrightarrow{W^{(k)}}\|_{\infty}\right)+1-\alpha
\end{equation*}
If we can bound the two terms: $T\|\overrightarrow{R}\|_{\infty},\|\sum_{k=1}^KT^{(k)}\overrightarrow{W^{(k)}}\|_{\infty}\le \frac{\alpha\lambda}{8}$, then we will have:
$$
\|\tilde{Z}_{S^c}\|_{\infty}\le1-\frac{\alpha}{2}<1
$$
From all the reasoning so far, we have the following Lemma:
\begin{lemma} \label{lem:twoBounds}
	If we have $T\|\overrightarrow{R(\Psi)}\|_{\infty},\|\sum_{k=1}^KT^{(k)}\overrightarrow{W^{(k)}}\|_{\infty}\le \frac{\alpha\lambda}{8}$, then
	$$
	\|\tilde{Z}_{S^c}\|_{\infty}<1\text{,}
	$$
	i.e., the strict-dual feasibility condition is fulfilled.
\end{lemma}
Thus the key step is to bound $T\|\overrightarrow{R}\|_{\infty}$ and $\|\sum_{k=1}^KT^{(k)}\overrightarrow{W^{(k)}}\|_{\infty}$ by 
$\frac{\alpha\lambda}{8}$. We will first consider $T\|\overrightarrow{R}\|_{\infty}$.

We have the following Lemma in \cite{ravikumar2011high} (Lemma 5):
\begin{lemma} \label{lem:1}
	For any $\rho\in \mathbb{R}^{N\times N}$, If we have $\|\rho\|_{\infty}\le \frac{1}{3}\kappa_{\bar{\Sigma}}d$, then the matrix 
	$J(\rho):=\sum_{k=0}^{\infty}(-1)^k(\bar{\Omega}^{-1}\rho)^k$ will satisfy $\vertiii{J^{\text{T}}}_{\infty}\le\frac{3}{2}$ and the matrix $R(\rho):=(\bar{\Omega}+\rho)^{-1}-\bar{\Omega}^{-1}+\bar{\Omega}^{-1}\rho\bar{\Omega}^{-1}$ will satisfy:
	\begin{equation} \label{eq:lem1.1}
	R(\rho)=\bar{\Omega}^{-1}\rho\bar{\Omega}^{-1}\rho J(\rho)\bar{\Omega}^{-1}
	\end{equation}
	and
	\begin{equation} \label{eq:lem1.2}
	\|R(\rho)\|_{\infty}\le \frac{3}{2}d\|\rho\|_{\infty}^2\kappa_{\bar{\Sigma}}^3
	\end{equation}
	Here
	$\kappa_{\bar{\Sigma}}:=\vertiii{\bar{\Sigma}}_{\infty}=\vertiii{\bar{\Omega}^{-1}}_{\infty}$, 
	$d:=\max_{1\le i\le N}\#\left\{j:1\le j\le N,\bar{\Omega}_{ij}\neq0\right\}$.
\end{lemma}

For $R(\rho)$ defined in the above Lemma, we vectorize $R(\rho)_S$ and then we have
\begin{equation} \label{eq:Rrho}
\begin{aligned} 
\overrightarrow{R(\rho)_S}=&\text{vec}\left(\left[(\bar{\Omega}+\rho)^{-1}-\bar{\Omega}^{-1}\right]_S\right)+
\text{vec}\left([\bar{\Omega}^{-1}\rho\bar{\Omega}^{-1}]_S\right)\\=&
\text{vec}\left(\left[(\bar{\Omega}+\rho)^{-1}]_S-[\bar{\Omega}^{-1}\right]_S\right)+\bar{\Gamma}_{SS}\overrightarrow{\rho_S}
\end{aligned}
\end{equation}
where the first line follows from the definition of $R(\rho)$ in Lemma \ref{lem:1} and the second line follows from \eqref{eq:vecGamma}

Define $\kappa_{\bar{\Gamma}}:=\vertiii{\bar{\Gamma}_{SS}^{-1}}_{\infty}$. For $\Omega\in\mathbb{R}^{N\times N}$, define the subgradient of \eqref{eq:step2} as $G(\Omega_{S})$, i.e., $G(\Omega_{S}):=-T[\Omega^{-1}]_S+\sum_{k=1}^KT^{(k)}\hat{\Sigma}^{(k)}_S+\lambda\tilde{Z}_S$.
Since we have proved in Step 1 that $\ell$ is strictly convex, $\tilde{\Omega}_S$ is the only solution of the restricted problem of \eqref{eq:step2}. Therefore $\tilde{\Omega}_S$ is the only solution that satisfies the stationary condition $G(\Omega_S)=0$.

Next for $\rho\in\mathbb{R}^{N\times N}$, define 
$F(\overrightarrow{\rho_{S}})=-\frac{1}{T}\bar{\Gamma}_{SS}^{-1}\overrightarrow{G}(\bar{\Omega}_S+\rho_S)+\overrightarrow{\rho_S}$. Then:
$$
F(\overrightarrow{\rho_{S}})=\overrightarrow{\rho_{S}} \Leftrightarrow G(\bar{\Omega}_S+\rho_S)=0
\Leftrightarrow \bar{\Omega}_S+\rho_S=\tilde{\Omega}_S
$$
Thus the fixed point of $F(\cdot)$ is $\Psi_{S}=\tilde{\Omega}_S-\bar{\Omega}_S$ and it is unique. 

Now define 
$r:=2\kappa_{\bar{\Gamma}}\left(\frac{\lambda}{T}+\|\sum_{k=1}^K\frac{T^{(k)}}{T}W^{(k)}\|_{\infty}\right)$. 
Suppose $r\le\min\left\{\frac{1}{3\kappa_{\bar{\Sigma}}d},\frac{1}{3\kappa_{\bar{\Sigma}}^3\kappa_{\bar{\Gamma}}d}\right\}$. 
Define the $\ell_{\infty}$ radius-$r$ ball $\mathbb{B}(r):=\left\{\rho_S:\|\rho_S\|_{\infty}\le r\right\}$. For $\forall \rho_S\in\mathbb{B}(r)$, define $\rho=(\rho_S,0)$, i.e., $[\rho]_S=\rho_S$ and $[\rho]_{S^c}=0$. We have:
$$
G(\bar{\Omega}_S+\rho_S)=T\left(-[(\bar{\Omega}+\rho)^{-1}]_S+[\bar{\Omega}^{-1}]_S\right)+\sum_{k=1}^KT^{(k)}W^{(k)}_S
+\lambda\tilde{Z}_S
$$
Then,
\begin{equation} \label{eq:defV1V2}
\begin{aligned}
F(\overrightarrow{\rho_S})=&-\frac{1}{T}\bar{\Gamma}_{SS}^{-1}\text{vec}
\left(T\left(-[(\bar{\Omega}+\rho)^{-1}]_S+[\bar{\Omega}^{-1}]_S\right)+\sum_{k=1}^KT^{(k)}W^{(k)}_S
+\lambda\tilde{Z}_S\right) + \overrightarrow{\rho_{S}}\\=&
\bar{\Gamma}_{SS}^{-1}\left\{\text{vec}
\left([(\bar{\Omega}+\rho)^{-1}]_S-[\bar{\Omega}^{-1}]_S\right)+\bar{\Gamma}_{SS}\overrightarrow{\rho_{S}}\right\}
-\frac{1}{T}\bar{\Gamma}_{SS}^{-1}\text{vec}\left(\sum_{k=1}^KT^{(k)}W^{(k)}_S+\lambda\tilde{Z}_S\right) \\=&
\underbrace{\bar{\Gamma}_{SS}^{-1}\overrightarrow{R(\rho)_S}}_{V_1}-
\underbrace{\frac{1}{T}\bar{\Gamma}_{SS}^{-1}
	\left(\sum_{k=1}^KT^{(k)}\overrightarrow{W^{(k)}_S}+\lambda\overrightarrow{\tilde{Z}_{S}}\right)}_{V_2}
\end{aligned}
\end{equation}
where the third line follows from \eqref{eq:Rrho}.
For $V_2$ defined above we have:
\begin{equation*}
\begin{aligned}
\|V_2\|_{\infty}\le &\vertiii{\bar{\Gamma}_{SS}^{-1}}_{\infty}\|\frac{\lambda}{T}\overrightarrow{\tilde{Z}_{S}}+
\sum_{k=1}^K\frac{T^{(k)}}{T}W^{(k)}\|_{\infty} \\
\le &\kappa_{\bar{\Gamma}}\left(\frac{\lambda}{T}+\|\sum_{k=1}^K\frac{T^{(k)}}{T}W^{(k)}\|_{\infty}\right)\\=&\frac{r}{2}
\end{aligned}
\end{equation*}
where the first inequality follows from \eqref{eq:normIneq1}, the second inequality follows from \eqref{eq:step3ineq} and the third line follows from the definition of $r$.

For $V_1$ defined in \eqref{eq:defV1V2} we have:
\begin{equation} \label{eq:V1}
\begin{aligned}
\|V_1\|_{\infty}&\le \vertiii{\bar{\Gamma}_{SS}^{-1}}_{\infty}\|R(\rho)_S\|_{\infty}\\ & \le
\kappa_{\bar{\Gamma}}\|R(\rho)\|_{\infty}\\&\le \kappa_{\bar{\Gamma}}
\left(\frac{3}{2}d\kappa_{\bar{\Sigma}}^3\right)\|\rho\|_{\infty}^2\\&\le 
\frac{3}{2}d\kappa_{\bar{\Sigma}}^3\kappa_{\bar{\Gamma}}r^2\\&\le \frac{r}{2}
\end{aligned}
\end{equation}
where the first inequality is due to \eqref{eq:normIneq1} and the second inequality is due to Lemma \ref{lem:1} and $\|\rho\|_{\infty}=\|\rho_S\|_{\infty}\le r$.

Thus $\|F(\overrightarrow{\rho_S})\|_{\infty}\le r$, $F(\overrightarrow{\rho_S})\in\mathbb{B}(r)$, which indicates $F(\mathbb{B}(r))\subset\mathbb{B}(r)$. 
By Brouwer's fixed point theorem (see e.g., \cite{ortega2000iterative}), there exists some fixed point of $F(\cdot)$ in $\mathbb{B}(r)$. We have proved that the fixed point of $F(\cdot)$ is $\Psi_S$ and it is unique, therefore $\Psi_S\in\mathbb{B}(r)$, i.e., $\|\Psi\|_{\infty}=\|\Psi_S\|_{\infty}\le r$. Thus by Lemma \ref{lem:1}, $\|R(\Psi)\|_{\infty}\le\frac{3}{2}d\|\Psi\|_{\infty}^2\kappa_{\bar{\Sigma}}^3$.  

From all the reasoning so far, we have the following Lemma:
\begin{lemma} \label{lem:r_condition}
	If $r=2\kappa_{\bar{\Gamma}}\left(\frac{\lambda}{T}+\|\sum_{k=1}^K\frac{T^{(k)}}{T}W^{(k)}\|_{\infty}\right)\le\min\left\{\frac{1}{3\kappa_{\bar{\Sigma}}d},\frac{1}{3\kappa_{\bar{\Sigma}}^3\kappa_{\bar{\Gamma}}d}\right\}$, then
	$$
	\|\Psi\|_{\infty}\le r
	$$
	and
	$$
	\|R(\Psi)\|_{\infty}\le\frac{3}{2}d\|\Psi\|_{\infty}^2\kappa_{\bar{\Sigma}}^3
	$$
\end{lemma}

If $\|\sum_{k=1}^{K}\frac{T^{(k)}}{T}W^{(k)}\|_{\infty}\le\xi$ with $\xi>0$, 
then choosing $\lambda=\frac{8T\xi}{\alpha}$, we will have 
$$
\|\sum_{k=1}^{K}T^{(k)}W^{(k)}\|_{\infty}\le\frac{\alpha\lambda}{8}
$$
as well as
$$
r=2\kappa_{\bar{\Gamma}}\left(\frac{\lambda}{T}+\|\sum_{k=1}^K\frac{T^{(k)}}{T}W^{(k)}\|_{\infty}\right)
\le2\kappa_{\bar{\Gamma}}\left(\frac{8}{\alpha}+1\right)\xi
$$
For $\xi\le \delta^*:= \frac{\alpha^2}{2\kappa_{\bar{\Gamma}}(\alpha+8)^2}\min\left\{\frac{1}{3\kappa_{\bar{\Sigma}}d},\frac{1}{3\kappa_{\bar{\Sigma}}^3\kappa_{\bar{\Gamma}}d}\right\}$, we have
$r\le\min\left\{\frac{1}{3\kappa_{\bar{\Sigma}}d},\frac{1}{3\kappa_{\bar{\Sigma}}^3\kappa_{\bar{\Gamma}}d}\right\}$.
Thus according to Lemma \ref{lem:r_condition}, we have
$$
\|\Psi\|_{\infty}=\|\Psi_S\|_{\infty}\le r\le2\kappa_{\bar{\Gamma}}\left(\frac{8}{\alpha}+1\right)\xi
$$
Therefore,
$$
\begin{aligned}
\|R(\Psi)\|_{\infty}\le & \frac{3}{2}d\|\Psi\|_{\infty}^2\kappa_{\bar{\Sigma}}^3
\\ & \le 6d\kappa_{\bar{\Sigma}}^3\kappa_{\bar{\Gamma}}^2\left(\frac{8}{\alpha}+1\right)^2\delta^2\\ = &
\left(6d\kappa_{\bar{\Sigma}}^3\kappa_{\bar{\Gamma}}^2\left(\frac{8}{\alpha}+1\right)^2\xi\right)\frac{\alpha\lambda}{8T}
\\ & \le \frac{\alpha\lambda}{8T}
\end{aligned}
$$
Then by Lemma \ref{lem:twoBounds}, $\|\tilde{Z}_{S^c}\|_{\infty}<1$ and the strict dual feasibility condition is fulfilled. According to the primal-dual witness approach, $\text{supp}(\hat{\Omega})=\text{supp}(\tilde{\Omega})\subseteq\text{supp}\left(\bar{\Omega}\right)$. 

From all the reasoning so far, we can state the following lemma.
\begin{lemma} \label{lem:strictdual}
	If $\|\sum_{k=1}^{K}\frac{T^{(k)}}{T}W^{(k)}\|_{\infty}\le\xi$ with $\xi\in(0,\delta^*]$, then choosing $\lambda=\frac{8T\xi}{\alpha}$, we have $\hat{\Omega}=\tilde{\Omega}$,  $\text{supp}(\hat{\Omega})\subseteq\text{supp}\left(\bar{\Omega}\right)$ and
	$$
	\|\hat{\Omega}-\bar{\Omega}\|_{\infty}=\|\Psi\|_{\infty}\le2\kappa_{\bar{\Gamma}}\left(\frac{8}{\alpha}+1\right)\xi
	$$
\end{lemma}
For the next step, we need to prove the tail condition of $\sum_{k=1}^{K}\frac{T^{(k)}}{T}W^{(k)}$, that is, for $\xi>0$, 
$\|\sum_{k=1}^{K}\frac{T^{(k)}}{T}W^{(k)}\|_{\infty}\le\xi$ with high probability.

\subsection{Proof of the Tail condition}
Note that for $k=1,\dots,K$,
\begin{equation} \label{eq:IneqWk}
W^{(k)}=\hat{\Sigma}^{(k)}-\bar{\Sigma}=\hat{\Sigma}^{(k)}-\bar{\Sigma}^{(k)}+\bar{\Sigma}^{(k)}-\bar{\Sigma}=
\hat{\Sigma}^{(k)}-\bar{\Sigma}^{(k)}+\left(\bar{\Omega}+\Delta^{(k)}\right)^{-1}-\bar{\Sigma}
\end{equation}
Here $\{\Delta^{(k)}\}_{k=1}^K$ are i.i.d. random matrices following the distribution $P$ specified in Definition \ref{def:msGGM}. To achieve the tail condition of $\sum_{k=1}^{K}\frac{T^{(k)}}{T}W^{(k)}$, we can bound the random terms with respect to $\{\Delta^{(k)}\}_{k=1}^K$ and the random terms with respect to the empirical sample covariance matrices $\{\hat{\Sigma}^{(k)}\}_{k=1}^K$ separately.

We have assumed that the sample size is the same for all tasks, i.e., there are $n$ samples for each of the $K$ tasks and $T^{(k)}/T=1/K$. For the sample covariance matrices, we have the following lemma:
\begin{lemma} \label{lem:TailSampleCov}
	For $\{X_{t}^{(k)}\}_{1\le t\le n,1\le k\le K}$ following a family of random $N$-dimensional multivariate sub-Gaussian distributions of size $K$ with parameter $\sigma$ described in Definition \ref{def:msGGM}, we have
	\begin{equation} \label{eq:IneqSigmaNorm_ij}
	\begin{aligned}
	\mathbb{P}\left[\left|\sum_{k=1}^K\frac{1}{K}\left(\hat{\Sigma}^{(k)}_{ij}-\bar{\Sigma}^{(k)}_{ij}\right)\right|>\nu\right]\le
	\exp\left\{-\frac{nK\nu^2}{128\left(1+4\sigma^2\right)^2\gamma^2}\right\}
	\end{aligned}
	\end{equation}
	and
	\begin{equation} \label{eq:IneqSigmaNorm}
	\begin{aligned}
	\mathbb{P}\left[\|\sum_{k=1}^K\frac{1}{K}\left(\hat{\Sigma}^{(k)}-\bar{\Sigma}^{(k)}\right)\|_{\infty}>\nu\right]\le
	2N(N+1)\exp\left\{-\frac{nK\nu^2}{128\left(1+4\sigma^2\right)^2\gamma^2}\right\}
	\end{aligned}
	\end{equation}
	for $\hat{\Sigma}^{(k)}=\frac{1}{n}\sum_{t=1}^nX_t^{(k)}(X_t^{(k)})^{\text{T}}$, $1\le i,j\le N$, and  $0\le\nu\le8\left(1+4\sigma^2\right)\gamma$.
\end{lemma}
The proof of this lemma is in Section \ref{sec:tailsample}.

For $\{\Delta^{(1)}\}_{k=1}^K$, we have the following lemma
\begin{lemma} \label{lem:Taildelta}
	For $\{\Delta^{(k)}\}_{k=1}^K$ in a family of random $N$-dimensional multivariate sub-Gaussian distributions of size $K$ with parameter $\sigma$ described in Definition \ref{def:msGGM}, define
	\begin{align}
	H(\Delta^{(1)},\dots,\Delta^{(K)}):=
	\frac{1}{K}\sum_{k=1}^K\bar{\Sigma}^{(k)}=\frac{1}{K}\sum_{k=1}^K\left(\bar{\Omega}+\Delta^{(k)}\right)^{-1}
	\end{align}
	Then we have
	\begin{equation} \label{eq:Ineqdelta}
	\begin{aligned}
	\mathbb{P}\left[\vertiii{H-\mathbb{E}[H]}_{2}>t\right]\le
	2N\exp\left\{-\frac{\lambda_{\min}^4Kt^2}{128c_{\max}^2}\right\}
	\end{aligned}
	\end{equation}
	for $t\ge0$ and $\lambda_{\min}=\lambda_{\min}(\bar{\Omega})$.
\end{lemma}
The proof of this lemma is in Section \ref{sec:taildelta}.

Our goal is to find a probability upper bound for $\|\sum_{k=1}^K\frac{T^{(k)}}{T}W^{(k)}\|_{\infty}>\xi$ with $0<\xi\le\delta^*$. According to \eqref{eq:IneqWk} and the condition $\beta\le\delta^*/2$, we have
\begin{equation} \label{eq:Wknorm}
\begin{aligned}
\|\sum_{k=1}^K\frac{T^{(k)}}{T}W^{(k)}\|_{\infty}=&\|\sum_{k=1}^K\frac{1}{K}W^{(k)}\|_{\infty} \\ \le &
\|\sum_{k=1}^K\frac{1}{K}\left(\hat{\Sigma}^{(k)}-\bar{\Sigma}^{(k)}\right)\|_{\infty}+
\|\frac{1}{K}\sum_{k=1}^{K}\bar{\Sigma}^{(k)}-\bar{\Sigma}\|_{\infty} \\ = &
\|\sum_{k=1}^K\frac{1}{K}\left(\hat{\Sigma}^{(k)}-\bar{\Sigma}^{(k)}\right)\|_{\infty}+
\|H-\mathbb{E}[H]+\mathbb{E}[H]-\bar{\Sigma}\|_{\infty} \\ = &
\|\sum_{k=1}^K\frac{1}{K}\left(\hat{\Sigma}^{(k)}-\bar{\Sigma}^{(k)}\right)\|_{\infty}+
\vertiii{H-\mathbb{E}[H]}_2+\|\mathbb{E}[H]-\bar{\Sigma}\|_{\infty}
 \\ \le &
\|\sum_{k=1}^K\frac{1}{K}\left(\hat{\Sigma}^{(k)}-\bar{\Sigma}^{(k)}\right)\|_{\infty}+\vertiii{H-\mathbb{E}[H]}_2+\beta
\end{aligned}
\end{equation}
where we have used the property that $\vertiii{A}_2\ge \|A\|_{\infty}$ for any matrix $A$ (see e.g., \cite{horn2012matrix}).

Now for $\delta\in(0,\delta^*/2]$, consider
\begin{equation} \label{eq:defDeltaR}
\xi=\delta+\delta^*/2
\end{equation}
then $0<\xi\le \delta^*$, $\delta+\tau\le\xi$ and $\lambda=\frac{8T\xi}{\alpha}=\frac{8\delta+4\delta^*}{\alpha}$.

According to the condition $\beta\le\delta^*/2$, we know that $\delta^*/2-\beta\ge0$. Set $t=\delta^*/2-\beta$ in \eqref{eq:Ineqdelta}. Then,
\begin{equation} \label{eq:Ineqdelta1}
\begin{aligned}
\mathbb{P}\left[\vertiii{H-\mathbb{E}[H]}_{2}>\delta^*/2-\beta\right]\le
2N\exp\left(-\frac{\lambda_{\min}^4K}{128c_{\max}^2}\left(\frac{\delta^*}{2}-\beta\right)^2\right)
\end{aligned}
\end{equation}

By \eqref{eq:Wknorm} and \eqref{eq:defDeltaR}, we have
$$
\left\{\|\sum_{k=1}^K\frac{1}{K}\left(\hat{\Sigma}^{(k)}-\bar{\Sigma}^{(k)}\right)\|_{\infty}\le \delta \text{ and }
\vertiii{H-\mathbb{E}[H]}_{2}\le \frac{\delta^*}{2}-\beta\right\} \Rightarrow
\left\{\|\sum_{k=1}^K\frac{T^{(k)}}{T}W^{(k)}\|_{\infty}\le \xi\right\}
$$
and thus
\begin{equation} \label{eq:WandSigma}
\begin{aligned}
\mathbb{P}\left[\|\sum_{k=1}^K\frac{T^{(k)}}{T}W^{(k)}\|_{\infty}\le \xi\right] & \ge
\mathbb{P}\left[\|\sum_{k=1}^K\frac{1}{K}\left(\hat{\Sigma}^{(k)}-\bar{\Sigma}^{(k)}\right)\|_{\infty}\le\delta \text{ and }
\vertiii{H-\mathbb{E}[H]}_{2}\le \frac{\delta^*}{2}-\beta\right] \\ &=
1-\mathbb{P}\left[\|\sum_{k=1}^K\frac{1}{K}\left(\hat{\Sigma}^{(k)}-\bar{\Sigma}^{(k)}\right)\|_{\infty}>\delta \text{ or }
\vertiii{H-\mathbb{E}[H]}_{2}> \frac{\delta^*}{2}-\beta\right]\\ &\ge 1-
\left(\mathbb{P}\left[\|\sum_{k=1}^K\frac{1}{K}\left(\hat{\Sigma}^{(k)}-\bar{\Sigma}^{(k)}\right)\|_{\infty}>\delta\right]+
\mathbb{P}\left[\vertiii{H-\mathbb{E}[H]}_{2}> \frac{\delta^*}{2}-\beta\right]\right)\\&=
1-\mathbb{P}\left[\|\sum_{k=1}^K\frac{1}{K}\left(\hat{\Sigma}^{(k)}-\bar{\Sigma}^{(k)}\right)\|_{\infty}>\delta\right]
-2N\exp\left(-\frac{\lambda_{\min}^4K}{128c_{\max}^2}\left(\frac{\delta^*}{2}-\beta\right)^2\right)
\end{aligned}
\end{equation}
where we have applied \eqref{eq:Ineqdelta1} for the last step.

When $0<\delta<8\left(1+4\sigma^2\right)\gamma$, we can let $\nu=\delta$ in \eqref{eq:IneqSigmaNorm} to get
\begin{equation} \label{eq:IneqSigmaNorm1}
\begin{aligned}
\mathbb{P}\left[\|\sum_{k=1}^K\frac{1}{K}\left(\hat{\Sigma}^{(k)}-\bar{\Sigma}^{(k)}\right)\|_{\infty}>\delta\right]\le 1-
2N(N+1)\exp\left\{-\frac{nK\delta^2}{128\left(1+4\sigma^2\right)^2\gamma^2}\right\}
\end{aligned}
\end{equation}

When $\delta\ge 8\left(1+4\sigma^2\right)\gamma$, we set $\nu=8\left(1+4\sigma^2\right)\gamma$ in \eqref{eq:IneqSigmaNorm} to get
\begin{equation} \label{eq:IneqSigmaNorm2}
\begin{aligned}
\mathbb{P}\left[\|\sum_{k=1}^K\frac{1}{K}\left(\hat{\Sigma}^{(k)}-\bar{\Sigma}^{(k)}\right)\|_{\infty}>\delta\right]\le &
\mathbb{P}\left[\|\sum_{k=1}^K\frac{1}{K}\left(\hat{\Sigma}^{(k)}-\bar{\Sigma}^{(k)}\right)\|_{\infty}>
8\left(1+4\sigma^2\right)\gamma\right]\\ 
\le & 2N(N+1)\exp\left\{-\frac{nK(8\left(1+4\sigma^2\right)\gamma)^2}
{128\left(1+4\sigma^2\right)^2\gamma^2}\right\}\\ = &
2N(N+1)\exp\left\{-\frac{nK}{2}\right\}
\end{aligned}
\end{equation}

Consider the maximum value of the two upper bounds in \eqref{eq:IneqSigmaNorm1} and \eqref{eq:IneqSigmaNorm2}. We can get
\begin{equation} \label{eq:IneqSigmaNorm3}
\begin{aligned}
\mathbb{P}\left[\|\sum_{k=1}^K\frac{1}{K}\left(\hat{\Sigma}^{(k)}-\bar{\Sigma}^{(k)}\right)\|_{\infty}>\delta\right]\le&
\max\left\{2N(N+1)\exp\left\{-\frac{nK\delta^2}
{128\left(1+4\sigma^2\right)^2\gamma^2}\right\},2N(N+1)\exp\left\{-\frac{nK}{2}\right\}\right\}\\=&
2N(N+1)\exp\left(-\frac{nK}{2}\min\left\{\frac{\delta^2}{64\left(1+4\sigma^2\right)^2\gamma^2},1\right\}\right)
\end{aligned}
\end{equation}

According to \eqref{eq:WandSigma} and \eqref{eq:IneqSigmaNorm3}, we have
\begin{equation} \label{eq:Wk_avg}
\begin{aligned}
\mathbb{P}\left[\|\sum_{k=1}^K\frac{T^{(k)}}{T}W^{(k)}\|_{\infty}\le \xi\right] \ge
1-&\mathbb{P}\left[\|\sum_{k=1}^K\frac{1}{K}\left(\hat{\Sigma}^{(k)}-\bar{\Sigma}^{(k)}\right)\|_{\infty}>\delta\right]
-2N\exp\left(-\frac{\lambda_{\min}^4K}{128c_{\max}^2}\left(\frac{\delta^*}{2}-\beta\right)^2\right)\\ \ge
1-&2N(N+1)\exp\left(-\frac{nK}{2}\min\left\{\frac{\delta^2}{64\left(1+4\sigma^2\right)^2\gamma^2},1\right\}\right)\\-&
2N\exp\left(-\frac{\lambda_{\min}^4K}{128c_{\max}^2}\left(\frac{\delta^*}{2}-\beta\right)^2\right)
\end{aligned}
\end{equation}
Namely, with probability at least
$$
1-2N(N+1)\exp\left(-\frac{nK}{2}\min\left\{\frac{\delta^2}{64\left(1+4\sigma^2\right)^2\gamma^2},1\right\}\right)-
2N\exp\left(-\frac{\lambda_{\min}^4K}{128c_{\max}^2}\left(\frac{\delta^*}{2}-\beta\right)^2\right)
$$
we have $\|\sum_{k=1}^K\frac{T^{(k)}}{T}W^{(k)}\|_{\infty}\le \xi\le \delta^*$, 
$\text{supp}(\hat{\Omega})\subseteq\text{supp}\left(\bar{\Omega}\right)$ and according to Lemma \ref{lem:strictdual}, we have
$$
\|\hat{\Omega}-\bar{\Omega}\|_{\infty}=\|\Delta\|_{\infty}\le2\kappa_{\bar{\Gamma}}\left(\frac{8}{\alpha}+1\right)\xi=
\kappa_{\bar{\Gamma}}\left(\frac{8}{\alpha}+1\right)(2\delta+\delta^*)
$$
which completes our proof of Theorem \ref{thm:subset}.

\section{Proof of Theorem \ref{thm:consis}}
We have the following lemma as a sufficient condition for the sign-consistency of \eqref{eq:estimator}.
\begin{lemma} \label{lem:suffconsis}
	For $\xi\in(0,\delta^*]$, if 
	\begin{equation}
	\|\sum_{k=1}^K\frac{T^{(k)}}{T}W^{(k)}\|_{\infty}\le xi
	\end{equation}
	and
	\begin{equation} \label{eq:omega1}
	\frac{\omega_{\min}}{2}\ge2\kappa_{\bar{\Gamma}}\left(\frac{8}{\alpha}+1\right)\xi
	\end{equation}
	where $\omega_{\min}:=\min_{(i,j)\in S}|\bar{\Omega}_{ij}|$, then the estimate $\hat{\Omega}$ of \eqref{eq:estimator} is sign-consistent.
\end{lemma}
The proof is in Section \ref{sec:suffconsis}.

In the remaining part of the proof, we assume that the condition $\beta\le\delta^{\dagger}/2$ stated in Theorem \ref{thm:consis} is satisfied.
We will consider two cases for different $\omega_{\min}>0$.

Case (i). If
\begin{equation} \label{eq:omega2}
\omega_{\min}\ge\frac{2\alpha}{8+\alpha}
\min\left\{\frac{1}{3\kappa_{\bar{\Sigma}}d},\frac{1}{3\kappa_{\bar{\Sigma}}^3\kappa_{\bar{\Gamma}}d}\right\}
\end{equation}
then
$$
0<\delta^{\dagger}=\delta^*
$$
and
\begin{equation*}
\frac{\omega_{\min}}{2}\ge2\kappa_{\bar{\Gamma}}\left(\frac{8}{\alpha}+1\right)\delta^*
\end{equation*}
Thus for $\xi=\delta^*$, \eqref{eq:omega1} holds. Then according to \eqref{eq:Wk_avg}, with probability at least
\begin{align*}
1-2N(N+1)\exp\left(-\frac{nK}{2}\min\left\{\frac{(\delta^*/2)^2}{64\left(1+4\sigma^2\right)^2\gamma^2},1\right\}\right)-
2N\exp\left(-\frac{\lambda_{\min}^4K}{128c_{\max}^2}\left(\frac{\delta^*}{2}-\beta\right)^2\right)
\\=
1-2N(N+1)\exp\left(-\frac{nK}{2}\min\left\{\frac{(\delta^{\dagger})^2}{256\left(1+4\sigma^2\right)^2\gamma^2},1\right\}\right)-
2N\exp\left(-\frac{\lambda_{\min}^4K}{128c_{\max}^2}\left(\frac{\delta^*}{2}-\beta\right)^2\right)
\end{align*}
we have $\|\sum_{k=1}^K\frac{T^{(k)}}{T}W^{(k)}\|_{\infty}\le \delta^*$ and thus by Lemma \ref{lem:suffconsis}, we have that \eqref{eq:estimator} is sign-consistent.

Case (ii). If
\begin{equation*}
\omega_{\min}<\frac{2\alpha}{8+\alpha}
\min\left\{\frac{1}{3\kappa_{\bar{\Sigma}}d},\frac{1}{3\kappa_{\bar{\Sigma}}^3\kappa_{\bar{\Gamma}}d}\right\},
\end{equation*}
then
\begin{equation*}
\frac{\omega_{\min}}{2}<2\kappa_{\bar{\Gamma}}\left(\frac{8}{\alpha}+1\right)\delta^*
\end{equation*}
and
$$
0<\delta^{\dagger}=\delta^{\prime}\le\delta^*
$$
Thus
\begin{equation} \label{eq:omega3}
\frac{\omega_{\min}}{2}\ge2\kappa_{\bar{\Gamma}}\left(\frac{8}{\alpha}+1\right)\delta^{\prime}
\end{equation}
Now apply \eqref{eq:Wk_avg} with $\xi=\delta^{\prime}=\delta^{\dagger}$, we have
\begin{equation*}
\begin{aligned}
\mathbb{P}\left[\|\sum_{k=1}^K\frac{T^{(k)}}{T}W^{(k)}\|_{\infty}\le \delta^{\prime}\right]\ge 1-&
2N(N+1)\exp\left(-\frac{nK}{2}\min\left\{\frac{(\delta^{\prime}-\delta^{*}/2)^2}{64\left(1+4\sigma^2\right)^2\gamma^2},1\right\}\right)
\\-&2N\exp\left(-\frac{\lambda_{\min}^4K}{128c_{\max}^2}\left(\frac{\delta^*}{2}-\beta\right)^2\right)
\\\ge1-&2N(N+1)\exp\left(-\frac{nK}{2}\min\left\{\frac{(\delta^{\dagger}-\delta^{\dagger}/2)^2}
{64\left(1+4\sigma^2\right)^2\gamma^2},1\right\}\right)\\
-&2N\exp\left(-\frac{\lambda_{\min}^4K}{128c_{\max}^2}\left(\frac{\delta^*}{2}-\beta\right)^2\right) \\=1-&
2N(N+1)\exp\left(-\frac{nK}{2}\min\left\{\frac{(\delta^{\dagger})^2}{256\left(1+4\sigma^2\right)^2\gamma^2},1\right\}\right)
\\-&2N\exp\left(-\frac{\lambda_{\min}^4K}{128c_{\max}^2}\left(\frac{\delta^*}{2}-\beta\right)^2\right)
\end{aligned}
\end{equation*}
Therefore with probability at least 
\begin{align*}
1-2N(N+1)\exp\left(-\frac{nK}{2}\min\left\{\frac{(\delta^{\dagger})^2}{256\left(1+4\sigma^2\right)^2\gamma^2},1\right\}\right)
-2N\exp\left(-\frac{\lambda_{\min}^4K}{128c_{\max}^2}\left(\frac{\delta^*}{2}-\beta\right)^2\right)
\end{align*}
we have $\|\sum_{k=1}^K\frac{T^{(k)}}{T}W^{(k)}\|_{\infty}\le\delta^{\prime}$ and thus by Lemma \ref{lem:suffconsis}, sign-consistency is guaranteed.

In conclusion, when $\tau\le\delta^{\dagger}/2$,
with probability at least
\begin{equation*}
1-2N(N+1)\exp\left(-\frac{nK}{2}\min\left\{\frac{(\delta^{\dagger})^2}{256(1+4\sigma^2)^2\gamma^2},1\right\}\right)
-2N\exp\left(-\frac{\lambda_{\min}^4K}{128c_{\max}^2}\left(\frac{\delta^*}{2}-\beta\right)^2\right)
\end{equation*}
the estimator $\hat{\Omega}$ is sign-consistent and thus $\text{supp}(\hat{\Omega})=\text{supp}\left(\bar{\Omega}\right)$, which completes our proof of Theorem \ref{thm:consis}.

\section{Proof of Theorem \ref{thm:fano}}
For $\forall Q\in [-1/(2d), 1/(2d)]^{N\times N}$, let $\Omega(E):=I+Q\odot \text{mat}(E)$ for $E\in\mathcal{E}$ where $\mathcal{E}$ is the set of all possible values of $E$ generated according to Theorem \ref{thm:fano} and $\text{mat}(E)\in \{0,1\}^{N \times N}$ is defined as follows: $\text{mat}(E)_{ij}=1$ if $(i,j)\in E$ and $\text{mat}(E)_{ij}=0$ if $(i,j)\notin E$ for $\forall E\in\mathcal{E}$.
Then we know $\Omega(E)$ is real and symmetric. Thus its eigenvalues are real.
By Gershgorin circle theorem \cite{golub2012matrix}, for any eigenvalue $\lambda$ of $\Omega(E)$, $\lambda$ lies in one of the Gershgorin circles, i.e., $|\lambda-\Omega(E)_{jj}|\leq \sum_{l\neq j} |\Omega(E)_{jl}|$ holds for some $j$. 
Since $\text{mat}(E)_{jj}=0$ and $|Q_{jl}|\le\frac{1}{2d}$ for $1\le l\le N$, we have $\Omega(E)_{jj}=1$ 
and $\sum_{l\neq j}|\Omega(E)_{jl}|\le d\cdot\frac{1}{2d}=\frac{1}{2}$. Thus $\lambda\in \left[\frac{1}{2}, \frac{3}{2}\right]$ and $\Omega(E)$ is positive definite. Thus, we have constructed a multiple Gaussian graphical model.
Now consider $\Omega(E)^{-1}$. Because any eigenvalue $\mu$ of $[\Omega(E)]^{-1}$ is the reciprocal of an eigenvalue of $\Omega(E)$, we have $|\mu|\in \left[\frac{2}{3}, 2\right]$.

Use $\lambda_1(A)$ to denote the largest eigenvalue of matrix $A$. for $E,E'\in\mathcal{E}$, according to Theorem H.1.d. in \cite{marshall2010matrix}, we have
$$
\lambda_1(\Omega(E')\Omega(E)^{-1})\le\lambda_1(\Omega(E'))\lambda_1(\Omega(E)^{-1})\le\frac{3}{2}\cdot2=3
$$
which gives us
\begin{equation} \label{eq:trSS}
\text{tr}\left(\Omega(E')\Omega(E)^{-1}\right)\leq N\lambda_1(\Omega(E')\Omega(E)^{-1})\le3N
\end{equation}

For $\textbf{Q}=\{Q^{(k)}\}_{k=1}^K$, we know that there is a bijection between $\mathcal{E}$ and the set of all circular permutations of nodes $V=\{1,...,N\}$. Thus $|\mathcal{E}|$, i.e., the size of $\mathcal{E}$, is the total number of circular permutations of $N$ elements, which is $C_E := (N-1)!/2$.
Since $E$ is uniformly distributed on $\mathcal{E}$, the entropy of $E$ given $\textbf{Q}$ is $H(E|\textbf{Q}) = \log C_E$.

Consider a family of $N$-dimensional random multivariate Gaussian distributions of size $K$ with covariance matrices $\{\bar{\Sigma}^{(k)}\}_{k=1}^K$ generated according to Theorem \ref{thm:fano}. We use $\textbf{X}:=\{X_t^{(k)}\}_{1\le t\le n,1\le k\le K}$ to denote the collection of $n$ samples from each of the $K$ distributions. 
Then for the mutual information $\mathbb{I}(\textbf{X}; E|\textbf{Q})$. We have the following bound:
\begin{equation} \label{eq:InforXS}
\begin{aligned}
\mathbb{I}(\textbf{X};E|\textbf{Q}) \leq &\frac{1}{C_E^2}\sum_E \sum_{E'} \mathbb{KL}(P_{\textbf{X}|E,\textbf{Q}} \| P_{\textbf{X}|E',\textbf{Q}}) \\
= &\frac{1}{C_E^2}\sum_E \sum_{E'} \sum_{k=1}^K \sum_{t=1}^n \mathbb{KL}(P_{X^{(k)}_t|E,Q^{(k)}} \| P_{X^{(k)}_t|E',Q^{(k)}}) \\
=& \frac{n}{C_E^2}\sum_E \sum_{E'} \sum_{k=1}^K \frac{1}{2}
\bigg[\text{tr}\left((I+Q^{(k)}\odot \text{mat}(E')) (I+Q^{(k)}\odot \text{mat}(E))^{-1}\right) \\
& - N + \log\frac{\text{det}(I+Q^{(k)}\odot \text{mat}(E))}{\text{det}(I+Q^{(k)}\odot \text{mat}(E'))}\bigg]
\end{aligned}
\end{equation}

Since the summation is taken over all $(E,E')$ pairs, the $\log$ term cancels with each other.
For the trace term, by \eqref{eq:trSS}, we have
\begin{equation} \label{eq:trSSk}
\text{tr}\left((I+Q^{(k)}\odot \text{mat}(E')) (I+Q^{(k)}\odot \text{mat}(E))^{-1}\right)\le 3N
\end{equation}
for $1\le k\le K$ and $E,E'\in\mathcal{E}$.
Putting \eqref{eq:trSSk} back to \eqref{eq:InforXS} gives 
\begin{equation}
\mathbb{I}(\textbf{X};E|\textbf{Q}) \leq \frac{n}{C_E^2}\sum_E \sum_{E'} \sum_{k=1}^K \frac{1}{2}(3N-N) = nNK
\end{equation}

For any estimate $\hat S$ of $S$, define $\hat E=\{(i,j):(i,j)\in \hat S,i\neq j\}$. Since $E\subseteq S$, we have $\mathbb{P}\{S\neq \hat{S}\}\ge \mathbb{P}\{E\neq \hat{E}\}$. Then by applying Theorem 1 in \cite{ghoshal2017information}, we get 
\begin{equation*}
\begin{aligned}
\mathbb{P}\{S\neq \hat{S}\} &\geq \mathbb{P}\{E\neq \hat{E}\}\\ & \ge
1 - \frac{\mathbb{I}(\textbf{X};S|\textbf{Q})+\log 2}{H(S|\textbf{Q})}\\ &
\ge 1-\frac{nNK + \log 2}{\log [(N-1)!/2]}
\end{aligned}
\end{equation*}

For $\log((N-1)!)$, we have:
\begin{align*}
\log((N-1)!)&=\sum_{i=1}^{N-1}\log i\\ &\ge \int_{1}^{N-1}\log x dx\\ &=(N-1)\log(N-1)-N+2\\&=(N-1)\log N+(N-1)\log\frac{N-1}{N}+2-N
\end{align*}
Since
$$
(N-1)\log\frac{N-1}{N}+2=2-(N-1)\log\left(1+\frac{1}{N-1}\right)\ge 2-1>0
$$
we have
$$
\log((N-1)!)\ge (N-1)\log N-N=N\log N-N-\log N
$$
$$
\log((N-1)!/2)=\log((N-1)!)-\log2\ge N\log N-N-\log2N
$$
For $N\ge 5$, $N\log N-N-\log2N>0$, thus we have
$$
\mathbb{P}\{S\neq \hat{S}\} \ge1-\frac{nNK + \log 2}{\log [(N-1)!/2]}\ge 1-\frac{nNK + \log 2}{N\log N-N-\log2N}
$$
which completes our proof of Theorem \ref{thm:fano}.

\section{Proof of Theorem \ref{thm:novel}}
By assumption, we have successfully recovered the true support union in the first step, i.e., $\text{supp}(\hat{\Omega})= S$.
Since there are constraints that $\text{supp}(\Omega)\subseteq\text{supp}(\hat{\Omega})= S$ and $\text{diag}(\Omega)=\text{diag}(\hat{\Omega})$ in \eqref{eq:estimatorNovel1}, we have
\begin{equation} \label{eq:loss_novel}
\begin{aligned}
\ell^{(K+1)}(\Omega)&=\langle\hat{\Sigma}^{(K+1)},\Omega\rangle-\log\det\left(\Omega\right) \\&=
\langle\hat{\Sigma}^{(K+1),S},\Omega\rangle-\log\det\left(\Omega\right)
\end{aligned}
\end{equation}
where $\hat{\Sigma}^{(K+1),S}:=\left(\hat{\Sigma}^{(K+1)}_S,0\right)$. 
Then the Lagrangian of the problem \eqref{eq:estimatorNovel1} is
\begin{equation} \label{eq:lagrangian}
\ell^{(K+1)}(\Omega)+\lambda\|\Omega\|_1+\langle\mu,\Omega\rangle+\langle\nu, \text{diag}(\Omega-\hat{\Omega})\rangle
\end{equation}
where $\mu\in\mathbb{R}^{N\times N}, \nu\in\mathbb{R}^{N}$ are the Lagrange multipliers satisfying $\mu_S=0$. Here we set $\mu=\bar{\Sigma}^{(K+1),S^c}=(\bar{\Sigma}^{(K+1)}_{S^c},0)$ and $\nu=\text{diag}(\bar{\Sigma}^{(K+1)}-\hat{\Sigma}^{(K+1)})$ in \eqref{eq:lagrangian}.
Define $W^{(K+1)}:=\bar{\Sigma}^{(K+1),S_{\text{off}}}-\hat{\Sigma}^{(K+1),S_{\text{off}}}$. With the primal-dual witness approach, we can get the following lemma similar to Lemma \ref{lem:strictdual}. 
\begin{lemma} \label{lem:novel1}
	Under Assumption \ref{assump:novel}, if $\|W^{(K+1)}\|_{\infty}\le\xi$ with $\xi\in(0,\delta^{(K+1),*}]$,
	then choosing $\lambda=\frac{8\xi}{\alpha^{(K+1)}}$, we have $\text{supp}(\hat{\Omega}^{(K+1)})\subseteq\text{supp}\left(\bar{\Omega}^{(K+1)}\right)$ and
	\begin{equation} \label{eq:novellem1}
	\|\hat{\Omega}^{(K+1)}-\bar{\Omega}^{(K+1)}\|_{\infty}\le
	2\kappa_{\bar{\Gamma}^{(K+1)}}\left(\frac{8}{\alpha^{(K+1)}}+1\right)\xi
	\end{equation}
\end{lemma}
The proof is in Section \ref{sec:lemnovel1}. 

By the definition of $W^{(K+1)}$, we know that $W^{(K+1)}_{S_{\text{off}}^c}=0$ and $W^{(K+1)}_{S_{\text{off}}}=[\hat{\Sigma}^{(K+1)}-\bar{\Sigma}^{(K+1)}]_{S_{\text{off}}}$. Thus $\|W^{(K+1)}\|_{\infty}=\|[\hat{\Sigma}^{(K+1)}-\bar{\Sigma}^{(K+1)}]_{S_{\text{off}}}\|_{\infty}$. 
Since we have assumed $\|\bar{\Sigma}^{(K+1)}\|_{\infty}\le \gamma^{(K+1)}$, 
according to Lemma \ref{lem:TailSampleCov} and the proof of \eqref{eq:IneqSigmaNorm3}, we have
\begin{equation} \label{eq:W_K+1_tail}
\begin{aligned}
\mathbb{P}\left[\|W^{(K+1)}\|_{\infty}\le \delta^{(K+1),\dagger}\right]=&
\mathbb{P}\left[\|\hat{\Sigma}^{(K+1)}-\bar{\Sigma}^{(K+1)}\|_{\infty}\le \delta^{(K+1),\dagger}\right] \\ \le& 1-2|S_{\text{off}}|\exp\left(-\frac{n^{(K+1)}}{2}\min\left\{\frac{(\delta^{(K+1),\dagger})^2} {64(1+4\sigma^2)^2(\gamma^{(K+1)})^2},1\right\}\right)
\end{aligned}
\end{equation}
because $S_{\text{off}}$ is symmetric.

Similar to Lemma \ref{lem:suffconsis}, we have the following lemma for the sign-consistency of $\hat{\Omega}^{(K+1)}$ in \eqref{eq:estimatorNovel1}.
\begin{lemma} \label{lem:suffconsis_novel}
	For $\xi\in(0,\delta^{(K+1),*}]$, if 
	\begin{equation}
	\|W^{(K+1)}\|_{\infty}\le \xi
	\end{equation}
	and
	\begin{equation} \label{eq:omega1_novel}
	\frac{\omega_{\min}^{(K+1)}}{2}\ge2\kappa_{\bar{\Gamma}^{(K+1)}}\left(\frac{8}{\alpha^{(K+1)}}+1\right)\xi
	\end{equation}
	where $\omega_{\min}:=\min_{(i,j)\in S}|\bar{\Omega}_{ij}|$, then the estimate $\hat{\Omega}^{(K+1)}$ in \eqref{eq:estimatorNovel1} is sign-consistent.
\end{lemma}
The proof is in Section \ref{sec:suffconsis_novel}.
Similar to the proof of Theorem \ref{thm:consis}, we consider two cases of $\omega_{\min}^{(K+1)}$.

Case (i). If
\begin{equation} \label{eq:omega2_novel}
\omega_{\min}^{(K+1)}\ge\frac{2\alpha^{(K+1)}}{8+\alpha^{(K+1)}}
\min\left\{\frac{1}{3\kappa_{\bar{\Sigma}^{(K+1)}}d^{(K+1)}},
\frac{1}{3\kappa_{\bar{\Sigma}^{(K+1)}}^3\kappa_{\bar{\Gamma}^{(K+1)}}d^{(K+1)}}\right\}
\end{equation}
then
$$
0<\delta^{(K+1),\dagger}=\delta^{(K+1),*}
$$
and
\begin{equation*}
\frac{\omega_{\min}^{(K+1)}}{2}\ge2\kappa_{\bar{\Gamma}^{(K+1)}}\left(\frac{8}{\alpha^{(K+1)}}+1\right)\delta^{(K+1),*}
\end{equation*}
Thus for $\xi=\delta^{(K+1),*}$, \eqref{eq:omega1_novel} holds. Then according to \eqref{eq:W_K+1_tail}, with probability at least
\begin{align*}
1 - 2|S_{\text{off}}|\exp\left(-\frac{n^{(K+1)}}{2}\min\left\{\frac{(\delta^{(K+1),\dagger})^2} {64(1+4\sigma^2)^2(\gamma^{(K+1)})^2},1\right\}\right)
\end{align*}
we have $\|W^{(K+1)}\|_{\infty}\le \delta=\delta^{(K+1),*}$ and thus by Lemma \ref{lem:suffconsis_novel}, we have that \eqref{eq:estimatorNovel1} is sign-consistent.

Case (ii). If
\begin{equation*}
\omega_{\min}^{(K+1)}<\frac{2\alpha^{(K+1)}}{8+\alpha^{(K+1)}}
\min\left\{\frac{1}{3\kappa_{\bar{\Sigma}^{(K+1)}}d^{(K+1)}},
\frac{1}{3\kappa_{\bar{\Sigma}^{(K+1)}}^3\kappa_{\bar{\Gamma}^{(K+1)}}d^{(K+1)}}\right\}
\end{equation*}
then
\begin{equation*}
\frac{\omega_{\min}^{(K+1)}}{2}<2\kappa_{\bar{\Gamma}^{(K+1)}}\left(\frac{8}{\alpha^{(K+1)}}+1\right)\delta^{(K+1),*}
\end{equation*}
and
$$
0<\delta^{(K+1),\dagger}=\delta^{(K+1),\prime}\le \delta^{(K+1),*}
$$
Then
\begin{equation} \label{eq:omega3_novel}
\frac{\omega_{\min}^{(K+1)}}{2}\ge2\kappa_{\bar{\Gamma}^{(K+1)}}\left(\frac{8}{\alpha^{(K+1)}}+1\right)\delta^{(K+1),\prime}
\end{equation}
For $\xi=\delta^{(K+1),\prime}=\delta^{(K+1),\dagger}$, \eqref{eq:omega1_novel} holds. Now according to \eqref{eq:W_K+1_tail}, with probability at least 
\begin{align*}
1 - 2|S_{\text{off}}|\exp\left(-\frac{n^{(K+1)}}{2}\min\left\{\frac{(\delta^{(K+1),\dagger})^2} {64(1+4\sigma^2)^2(\gamma^{(K+1)})^2},1\right\}\right)
\end{align*}
we have $\|W^{(K+1)}\|_{\infty}\le\delta^{(K+1),\prime}=\delta^{(K+1),\dagger}$ and thus by Lemma $\ref{lem:suffconsis_novel}$, sign-consistency is guaranteed.

In conclusion, with probability at least
\begin{equation*}
\begin{aligned}
1 - 2|S_{\text{off}}|\exp\left(-\frac{n^{(K+1)}}{2}\min\left\{\frac{(\delta^{(K+1),\dagger})^2} {64(1+4\sigma^2)^2(\gamma^{(K+1)})^2},1\right\}\right)
\end{aligned}
\end{equation*}
the estimator $\hat{\Omega}^{(K+1)}$ is sign-consistent and thus
$\text{supp}(\hat{\Omega}^{(K+1)})=\text{supp}\left(\bar{\Omega}^{(K+1)}\right)$, which completes our proof of Theorem \ref{thm:novel}.

\section{Proof of Theorem \ref{thm:fano_novel}}
For $\forall Q\in [-1/(N\log s), 1/(N\log s)]^{N\times N},\ E^{(K+1)}\in\mathcal{E}$, we know $\Omega(E^{(K+1)})=I+Q\odot \text{mat}(E^{(K+1)})$ is real and symmetric, where $\text{mat}(\cdot) \in \{0,1\}^{N \times N}$ is defined in the proof of Theorem \ref{thm:fano}. Thus its eigenvalues are real.
By Gershgorin circle theorem \cite{golub2012matrix}, for any eigenvalue $\lambda$ of $\Omega(E^{(K+1)})$, $\lambda$ lies in one of the Gershgorin circles, i.e., $|\lambda-\Omega(E^{(K+1)})_{jj}|\leq \sum_{l\neq j} |\Omega(E^{(K+1)})_{jl}|$ holds for some $j$. 
Since $\text{mat}(E^{(K+1)})_{jj}=0$ and $|Q_{jl}|\le 1/(N\log s)$ for $1\le l\le N$, we have $\Omega(E^{(K+1)})_{jj}=1$.
Meanwhile, there are at most $s/2$ non-zero elements in any row of $\text{mat}(E^{(K+1)})$ because $|E^{(K+1)}|\le s$ and $\text{mat}(E^{(K+1)})$ is symmetric. Thus $\sum_{l\neq j}|\Omega(E)_{jl}|\le\frac{s}{2N\log s}$. Then we have $\lambda\in \left[1-\frac{s}{2N\log s}, 1+\frac{s}{2N\log s}\right]$ and $\Omega(E^{(K+1)})$ is positive definite. Thus, we have constructed a Gaussian graphical model.
Now consider $\Omega(E^{(K+1)})^{-1}$. Because any eigenvalue $\mu$ of $\Omega(E^{(K+1)})^{-1}$ is the reciprocal of an eigenvalue of $\Omega(E^{(K+1)})$, we have $|\mu|\le 1/(1-\frac{s}{2N\log s})$.

For any $E^{(K+1)},\tilde E^{(K+1)}\in\mathcal{E}$, according to Theorem H.1.d. in \cite{marshall2010matrix}, we have
$$
\lambda_1(\Omega(\tilde E^{(K+1)})\Omega(E^{(K+1)})^{-1})\le\lambda_1(\Omega(\tilde E^{(K+1)}))\lambda_1(\Omega(E^{(K+1)})^{-1})\le
\frac{1+\frac{s}{2N\log s}}{1-\frac{s}{2N\log s}}
$$
which gives us
\begin{equation} \label{eq:trSS_novel}
\text{tr}\left(\Omega(\tilde E^{(K+1)})\Omega(E^{(K+1)})^{-1}\right)\leq N\lambda_1(\Omega(\tilde E^{(K+1)})\Omega(E^{(K+1)})^{-1})\le
N\frac{1+\frac{s}{2N\log s}}{1-\frac{s}{2N\log s}}
\end{equation}

According to the definition of $\mathcal{E}$, we know that $|\mathcal{E}|=2^{s/2}$. Since $E^{(K+1)}$ is uniformly distributed on $\mathcal{E}$, the entropy of $E^{(K+1)}$ given $Q$ is 
\begin{equation} \label{eq:logS_novel}
\begin{aligned} 
H(E^{(K+1)}|Q) = \log |\mathcal{E}|\ge \frac{s}{2}\log 2
\end{aligned}
\end{equation}

Now let $\textbf{X}:=\{X_t\}_{1\le t\le n}$ be the samples from a $N$-dimensional multivariate Gaussian distribution with covariance $\bar{\Sigma}$ generated according to Theorem \ref{thm:fano_novel}. For the mutual information $\mathbb{I}(\textbf{X}; E^{(K+1)}|Q)$, we have the following bound:
\begin{equation} \label{eq:InforXS_novel}
\begin{aligned}
\mathbb{I}(\textbf{X};E^{(K+1)}|Q) \leq &\frac{1}{|\mathcal{E}|^2}\sum_{E^{(K+1)}} \sum_{\tilde E^{(K+1)}} \mathbb{KL}(P_{\textbf{X}|E^{(K+1)},Q} \| P_{\textbf{X}|\tilde E^{(K+1)},Q}) \\
= &\frac{1}{|\mathcal{E}|^2}\sum_{E^{(K+1)}} \sum_{\tilde E^{(K+1)}} \sum_{t=1}^n 
\mathbb{KL}(P_{X_t|E^{(K+1)},Q} \| P_{X_t|\tilde E^{(K+1)},Q}) \\
=& \frac{n}{|\mathcal{E}|^2}\sum_{E^{(K+1)}} \sum_{\tilde E^{(K+1)}} \frac{1}{2}
\bigg[\text{tr}\left((I+Q\odot \text{mat}(\tilde E^{(K+1)})) (I+Q\odot \text{mat}(E^{(K+1)}))^{-1}\right) \\
& - N + \log\frac{\text{det}(I+Q\odot \text{mat}(E^{(K+1)}))}{\text{det}(I+Q\odot \text{mat}(\tilde E^{(K+1)}))}\bigg]
\end{aligned}
\end{equation}
Since the summation is taken over all $(E^{(K+1)},\tilde E^{(K+1)})$ pairs, the $\log$ term cancels with each other.
For the trace term, by \eqref{eq:trSS_novel}, we have
\begin{equation} \label{eq:trSSk_novel}
\text{tr}\left((I+Q\odot \text{mat}(\tilde E^{(K+1)})) (I+Q\odot \text{mat}(E^{(K+1)}))^{-1}\right)\le 
N\frac{1+\frac{s}{2N\log s}}{1-\frac{s}{2N\log s}}
\end{equation}
for $E^{(K+1)},\tilde E^{(K+1)}\in\mathcal{E}$.
Putting \eqref{eq:trSSk_novel} back to \eqref{eq:InforXS_novel} gives 
\begin{equation} \label{eq:mutualinfo_novel}
\begin{aligned}
\mathbb{I}(\textbf{X};E^{(K+1)}|Q) &\leq \frac{n}{|\mathcal{E}|^2}\sum_{E^{(K+1)}} \sum_{\tilde E^{(K+1)}} 
\frac{1}{2}\left(N\frac{1+\frac{s}{2N\log s}}{1-\frac{s}{2N\log s}}-N\right) \\&= 
\frac{ns}{2\log s}\frac{1}{1-\frac{s}{2N\log s}}
\\&\leq \frac{2ns}{\log s}
\end{aligned}
\end{equation}
According to our assumption that $4\le s\le N$. 

Define $\hat{E}^{(K+1)}:=\{(i,j)\in\hat{S}^{(K+1)}:i\neq j\}$.
By applying Theorem 1 in \cite{ghoshal2017information}, we get 
\begin{equation*}
\begin{aligned}
\mathbb{P}\{S^{(K+1)}\neq \hat{S}^{(K+1)}\} \ge & \mathbb{P}\{E^{(K+1)}\neq \hat{E}^{(K+1)}\} \\
\ge & 1 - \frac{\mathbb{I}(\textbf{X};E^{(K+1)}|Q)+\log 2}{H(E^{(K+1)}|Q)}
\\ \ge & 1-\frac{\frac{2ns}{\log s} + \log 2}
{\log|\mathcal{E}|}\\ = & 1-\frac{\frac{2ns}{\log s} + \log 2}
{\frac{s}{2}\log2} \\ = & 1-\frac{4n}{(\log2)(\log s)}-\frac{2}{s}
\end{aligned}
\end{equation*}
where the third inequality is by \eqref{eq:mutualinfo_novel}.

\section{Proof of Lemma \ref{lem:TailSampleCov}} \label{sec:tailsample}
We first prove the following lemma showing that \eqref{eq:IneqSigmaNorm_ij} and \eqref{eq:IneqSigmaNorm} hold for deterministic covariance matrices $\{\Sigma^{(k)}\}_{k=1}^K$. 
\begin{lemma} \label{lem:TailSampleCov_fixed}
	For $K$ deterministic matrices $\{\bar{\Sigma}^{(k)}\}_{k=1}^K$ and $\gamma\ge \|\bar{\Sigma}^{(k)}\|_{\infty}$ for $1\le k\le K$, consider the samples $\{X_{t}^{(k)}\}_{1\le t\le n,1\le k\le K}\subseteq\mathbb{R}^N$ satisfying the following conditions:
	
	(i) $\mathbb{E}\left[X_{t}^{(k)}\right]=0$, $\text{Cov}\left(X_{t}^{(k)}\right)=\bar{\Sigma}^{(k)}$ for $1\le t\le n,\ 1\le k\le K$;
	
	(ii) $\left\{X_{t}^{(k)}\right\}_{1\le t\le n,\ 1\le k\le K}$ are independent;
	
	(iii) $\frac{X_{t,i}^{(k)}}{\sqrt{\bar{\Sigma}^{(k)}_{ii}}}$ is sub-Gaussian with parameter $\sigma$ for $1\le i\le N,\ 1\le t\le n,\ 1\le k\le K$.
	
	Then for the empirical sample covariance matrices $\{\hat{\Sigma}^{(k)}\}_{k=1}^K$, \eqref{eq:IneqSigmaNorm_ij} and \eqref{eq:IneqSigmaNorm} hold for $1\le i,j\le N$ and $0\le\nu\le8\left(1+4\sigma^2\right)\gamma$.
\end{lemma}
\begin{proof}
First consider the element-wise tail condition. For $1\le i,j\le N$, we need to find an upper bound of the following probability:
\begin{equation} \label{eq:tailprob1}
\mathbb{P}\left[\Big|\frac{1}{nK}\sum_{k=1}^{K}\sum_{t=1}^n\left(X_{t,i}^{(k)}X_{t,j}^{(k)}
-\bar{\Sigma}_{ij}^{(k)}\right)\Big|>\nu\right]
\end{equation}
Let $s_i:=\max_{1\le k\le K}\bar{\Sigma}_{ii}^{(k)}$, $s_j:=\max_{1\le k\le K}\bar{\Sigma}_{jj}^{(k)}$, 
$\tilde{X}_{t,i}^{(k)}:=\frac{X_{t,i}^{(k)}}{\sqrt{s_i}}$, $\tilde{X}_{t,j}^{(k)}:=\frac{X_{t,j}^{(k)}}{\sqrt{s_j}}$,
$\tilde{\rho}_{ij}^{(k)}:=\frac{\bar{\Sigma}_{ij}^{(k)}}{\sqrt{s_is_j}}$. We have
\begin{equation*}
\text{\eqref{eq:tailprob1}}=\mathbb{P}\left[4
\Big|\sum_{k,t}\left(\tilde{X}_{t,i}^{(k)}\tilde{X}_{t,j}^{(k)}-\tilde{\rho}_{ij}^{(k)}\right)\Big|
>\frac{4nK\nu}{\sqrt{s_is_j}}\right]
\end{equation*}
Define $U^{(k)}_{t,ij}:=\tilde{X}_{t,i}^{(k)}+\tilde{X}_{t,j}^{(k)}$, $V^{(k)}_{t,ij}:=\tilde{X}_{t,i}^{(k)}-\tilde{X}_{t,j}^{(k)}$. 
Then for any $r\in\mathbb{R}$,
\begin{equation} \label{eq:Ineq_r}
4\sum_{k,t}\left(\tilde{X}_{t,i}^{(k)}\tilde{X}_{t,j}^{(k)}-\tilde{\rho}_{ij}^{(k)}\right)=
\sum_{k,t}\left\{\left(U^{(k)}_{t,ij}\right)^2-2\left(r+\tilde{\rho}^{(k)}_{ij}\right)\right\}-
\sum_{k,t}\left\{\left(U^{(k)}_{t,ij}\right)^2-2\left(r-\tilde{\rho}^{(k)}_{ij}\right)\right\}
\end{equation}
Thus
\begin{equation} \label{eq:Ineqstar1}
\begin{aligned}
\text{\eqref{eq:tailprob1}}\le& \mathbb{P}\left[\Big|\sum_{k,t}
\left\{\left(U^{(k)}_{t,ij}\right)^2-2\left(r+\tilde{\rho}^{(k)}_{ij}\right)\right\}\Big|
>\frac{2nK\nu}{\sqrt{s_is_j}}\right] \\&+ \mathbb{P}\left[\Big|\sum_{k,t}
\left\{\left(V^{(k)}_{t,ij}\right)^2-2\left(r-\tilde{\rho}^{(k)}_{ij}\right)\right\}\Big|
>\frac{2nK\nu}{\sqrt{s_is_j}}\right]
\end{aligned}
\end{equation}
Now define 
$$
Z^{(k)}_{t,ij}:=\left(U^{(k)}_{t,ij}\right)^2-2\left(r+\tilde{\rho}^{(k)}_{ij}\right)
$$
Applying the inequality $(a+b)^m\le 2^m(a^m+b^m)$ on $Z^{(k)}_{t,ij}$, we have
\begin{equation} \label{eq:Zktij}
\mathbb{E}\left[\big|Z^{(k)}_{t,ij}\big|^m\right]\le2^m\left\{\mathbb{E}\left[\big|U^{(k)}_{t,ij}\big|^{2m}\right]
+\left[2\left(1+\tilde{\rho}^{(k)}_{ij}\right)\right]^m\right\}
\end{equation}
Let $r^{(k)}_i:=\sqrt{\frac{\bar{\Sigma}^{(k)}_{ii}}{s_i}}$, $r^{(k)}_i:=\sqrt{\frac{\bar{\Sigma}^{(k)}_{ii}}{s_i}}$, then
$$
\tilde{X}^{(k)}_{t,i}=\bar{X}^{(k)}_{t,i}r^{(k)}_i,\ \ \tilde{X}^{(k)}_{t,j}=\bar{X}^{(k)}_{t,j}r^{(k)}_j
$$
where $\bar{X}^{(k)}_{t,i}:=\frac{X^{(k)}_{t,i}}{\sqrt{\bar{\Sigma}^{(k)}_{ii}}}$, $\bar{X}^{(k)}_{t,j}:=\frac{X^{(k)}_{t,j}}{\sqrt{\bar{\Sigma}^{(k)}_{jj}}}$. 

Assume that $\bar{X}^{(k)}_{t,i}$ is sub-Gaussian with parameter $\sigma$ for $\le i\le N,\ 1\le t\le n, 1\le k\le K$, and then we have
$$
\mathbb{E}\left[\exp\left(\lambda\tilde{X}^{(k)}_{t,i}\right)\right]=
\mathbb{E}\left[\exp\left(\lambda\bar{X}^{(k)}_{t,i}r^{(k)}_i\right)\right]\le 
\exp\left\{\frac{\lambda^2}{2}\sigma^2\left(r^{(k)}_i\right)^2\right\}
$$
which shows that $\tilde{X}^{(k)}_{t,i}$ is sub-Gaussian with parameter $\sigma r^{(k)}_i$. Then
$$
\begin{aligned}
\mathbb{E}\left[\exp\left(\lambda U^{(k)}_{t,ij}\right)\right]=&
\mathbb{E}\left[\exp\left(\lambda\tilde{X}^{(k)}_{t,i}\right)\exp\left(\lambda\tilde{X}^{(k)}_{t,j}\right)\right]\\ \le &
\mathbb{E}\left[\exp\left(2\lambda\tilde{X}^{(k)}_{t,i}\right)\right]^{\frac{1}{2}}
\mathbb{E}\left[\exp\left(2\lambda\tilde{X}^{(k)}_{t,j}\right)\right]^{\frac{1}{2}} \\ \le &
\exp\left\{\lambda^2\sigma^2\left[\left(r^{(k)}_i\right)^2+\left(r^{(k)}_j\right)^2\right]\right\}
\end{aligned}
$$
Therefore $U^{(k)}_{t,ij}$ is sub-Gaussian with parameter $\sigma^{(k)}_{ij}:=\sigma\sqrt{2\left[\left(r^{(k)}_i\right)^2+\left(r^{(k)}_j\right)^2\right]}$.
Similarly, we can prove that $V^{(k)}_{t,ij}$ is sub-Gaussian with parameter $\sigma^{(k)}_{ij}$ as well. Also note that
$\sigma^{(k)}_{ij}\le\sigma\sqrt{2(1+1)}=2\sigma$.

As it is well-known (see e.g., Lemma 1.4 in \cite{buldygin2000metric}), for a sub-Gaussian random variable $X$ with parameter $\sigma$, i.e., $X$ that satisfies 
$\mathbb{E}\left[e^{\lambda X}\right]\le\exp\left(\frac{\lambda^2\sigma^2}{2}\right)$, we have:
\begin{equation}
\mathbb{E}\left[|X|^s\right]\le 2\left(\frac{s}{e}\right)^{s/2}\sigma^s
\end{equation}
Apply this lemma on $U^{(k)}_{t,ij}$ with $s=2m,\ m\ge2$ and we get
$$
\mathbb{E}\left[\big|U^{(k)}_{t,ij}\big|^{2m}\right]\le 2\left(\frac{2m}{e}\right)^{m}\left(\sigma^{(k)}_{ij}\right)^{2m}
$$
According to the inequality $m!\ge \left(\frac{m}{e}\right)^m$, we have
$$
\mathbb{E}\left[\frac{\big|U^{(k)}_{t,ij}\big|^{2m}}{m!}\right]\le 2^{m+1}\left(\sigma^{(k)}_{ij}\right)^{2}
$$
Plug in \eqref{eq:Zktij} and we have
\begin{equation}
\begin{aligned}
\left(\frac{\mathbb{E}\left[\big|Z^{(k)}_{t,ij}\big|^{m}\right]}{m!}\right)^{\frac{1}{m}}\le & 2^{\frac{1}{m}}
\left\{\left[2^{2m+1}\left(\sigma^{(k)}_{ij}\right)^{2m}\right]^{\frac{1}{m}}+
\frac{4\left(r+\tilde{\rho}^{(k)}_{ij}\right)}{\left(m!\right)^{\frac{1}{m}}}\right\}\\ \le & \underbrace{2^{\frac{1}{m}}
	\left\{4\cdot2^{\frac{1}{m}}\left(\sigma^{(k)}_{ij}\right)^{2}+
	\frac{4\left(r+\tilde{\rho}^{(k)}_{ij}\right)}{\left(m!\right)^{\frac{1}{m}}}\right\}}_{h(m)}
\end{aligned}
\end{equation}
Note that $h(m)$ defined above decreases with $m$ and $\big|\tilde{\rho}^{(k)}_{ij}\big|\le 1$.

Since \eqref{eq:Ineq_r} holds for $\forall r\in\mathbb{R}$, we can choose $r=\frac{\left(r^{(k)}_i\right)^2+\left(r^{(k)}_j\right)^2}{2}$. Then we have $r<1$ and
$$
Z^{(k)}_{t,ij}:=\left(U^{(k)}_{t,ij}\right)^2-
\left(\left(r^{(k)}_i\right)^2+\left(r^{(k)}_j\right)^2+2\tilde{\rho}^{(k)}_{ij}\right)
$$
Thus
\begin{equation*}
\mathbb{E}\left[Z^{(k)}_{t,ij}\right]=0
\end{equation*}
and furthermore,
\begin{equation*}
\begin{aligned}
\sup_{m\ge2}\left(\frac{\mathbb{E}\left[\big|Z^{(k)}_{t,ij}\big|^{m}\right]}{m!}\right)^{\frac{1}{m}}\le & h(2)\\ =&
8\left(\sigma^{(k)}_{ij}\right)^{2}+4\left(r+\big|\tilde{\rho}^{(k)}_{ij}\big|\right) \\ \le &
8\left(1+\left(\sigma^{(k)}_{ij}\right)^{2}\right) \\ \le &
8\left(1+4\sigma^2\right)
\end{aligned}
\end{equation*}

Define $B:=8\left(1+4\sigma^2\right)$. If $X$ is a random variable such that $\mathbb{E}\left[X\right]=0$, 
$\left(\frac{\mathbb{E}\left[|X|^{m}\right]}{m!}\right)^{\frac{1}{m}}\le B$ for $m\ge2$, then
$$
\mathbb{E}\left[e^{\lambda X}\right]=\mathbb{E}\left[\sum_{k=0}^{\infty}\frac{X^k}{k!}\lambda^k\right]
=1+\sum_{k=1}^\infty\lambda^k\frac{\mathbb{E}\left[X^k\right]}{k!}\le1+\sum_{k=1}^\infty\left(\lambda B\right)^k
\le1+\frac{(\lambda B)^2}{1-|\lambda|B}
$$
when $|\lambda|<\frac{1}{B}$. Meanwhile,
$$
1+\frac{(\lambda B)^2}{1-|\lambda|B}\le\exp\left\{\frac{\lambda^2B^2}{1-|\lambda|B}\right\}\le\exp\left(2\lambda^2B^2\right)
$$
when $|\lambda|\le\frac{1}{2B}$.
Therefore for $|\lambda|\le\frac{1}{2B}$,
\begin{equation}
\mathbb{E}\left[e^{\lambda X}\right]\le\exp\left(2\lambda^2B^2\right)=\exp\left(\frac{\lambda^2(2B)^2}{2}\right)
\end{equation}
Then for $X_i,\ 1\le i\le n$ independent and satisfying $\mathbb{E}\left[X_i\right]=0$, $\left(\frac{\mathbb{E}\left[|X_i|^{m}\right]}{m!}\right)^{\frac{1}{m}}\le B$ when $m\ge2$, we can claim that for $0\le\epsilon\le 2B$,
\begin{equation} \label{eq:tailXiB}
\mathbb{P}\left[\Big|\sum_{i=1}^nX_i\Big|>n\epsilon\right]\le 2\exp\left(-\frac{n\epsilon^2}{8B^2}\right)
\end{equation}
In fact, for $0\le t\le\frac{1}{2B}$,
\begin{equation}
\begin{aligned}
\mathbb{P}\left[\sum_{i=1}^nX_i> n\epsilon\right]&\le\mathbb{P}\left[e^{t\sum_{i=1}^nX_i}\ge e^{tn\epsilon}\right]\\ &\le
e^{-tn\epsilon}\mathbb{E}\left[e^{t\sum_{i=1}^nX_i}\right]\\&=
\left(\prod_{i=1}^{n}\mathbb{E}\left[e^{tX_i}\right]\right)e^{-tn\epsilon} \\ & \le \exp\left(2nt^2B^2-tn\epsilon\right)
\end{aligned}
\end{equation}
Thus choosing $t=\frac{\epsilon}{4B^2}\le\frac{1}{2B}$, we can get
$$
\mathbb{P}\left[\sum_{i=1}^nX_i>n\epsilon\right]\le\exp\left(-\frac{n\epsilon^2}{8B^2}\right)
$$
Similarly, we can also prove that
$$
\mathbb{P}\left[\sum_{i=1}^nX_i<-n\epsilon\right]\le\exp\left(-\frac{n\epsilon^2}{8B^2}\right)
$$
Thus
$$
\mathbb{P}\left[\Big|\sum_{i=1}^nX_i\Big|>n\epsilon\right]=\mathbb{P}\left[\sum_{i=1}^nX_i>n\epsilon\right]+
\mathbb{P}\left[\sum_{i=1}^nX_i<-n\epsilon\right]\le2\exp\left(-\frac{n\epsilon^2}{8B^2}\right)
$$

Now consider $Z^{(k)}_{t,ij},\ 1\le t\le n,\ 1\le k\le K$. These random variables are independent by our assumption and satisfy $\mathbb{E}\left[Z^{(k)}_{t,ij}\right]=0$, $\sup_{m\ge2}\left(\frac{\mathbb{E}\left[\big|Z^{(k)}_{t,ij}\big|^{m}\right]}{m!}\right)^{\frac{1}{m}}\le8\left(1+4\sigma^2\right)=B$ by our proof. 
Then according to \eqref{eq:tailXiB}, for $0\le\frac{2\nu}{\gamma}\le 2B$, i.e., 
$0\le\nu\le 8\left(1+4\sigma^2\right)\gamma$, we have:
\begin{equation}
\begin{aligned}
\mathbb{P}\left[\Big|\sum_{k,t}Z^{(k)}_{t,ij}\Big|>\frac{2nK\nu}{\gamma}\right]
\le & 2\exp\left\{-\frac{4nK\nu^2}{8B^2\gamma^2}\right\} \\ = &
2\exp\left\{-\frac{nK\nu^2}{128\left(1+4\sigma^2\right)^2\gamma^2}\right\}
\end{aligned}
\end{equation}
Since $\gamma\ge \max_{1\le k\le K}\|\bar{\Sigma}^{(k)}\|_{\infty}=\max_{1\le l\le N}s_l\ge\sqrt{s_is_j}$ for $1\le i,j\le N$, we have:
\begin{equation}
\begin{aligned}
\mathbb{P}\left[\Big|\sum_{k,t}Z^{(k)}_{t,ij}\Big|>\frac{2nK\nu}{\sqrt{s_is_j}}\right]
\le\mathbb{P}\left[\Big|\sum_{k,t}Z^{(k)}_{t,ij}\Big|>\frac{2nK\nu}{\gamma}\right]
\le2\exp\left\{-\frac{nK\nu^2}{128\left(1+4\sigma^2\right)^2\gamma^2}\right\}
\end{aligned}
\end{equation}
Plug in the definition of $Z^{(k)}_{t,ij}$, we have
\begin{equation}
\begin{aligned}
\mathbb{P}\left[\Big|\sum_{k,t}\left\{\left(U^{(k)}_{t,ij}\right)^2-2\left(r+\tilde{\rho}^{(k)}_{ij}\right)\right\}\Big|
>\frac{2nK}{\sqrt{s_is_j}}\nu\right]
\le2\exp\left\{-\frac{nK\nu^2}{128\left(1+4\sigma^2\right)^2\gamma^2}\right\}
\end{aligned}
\end{equation}
Similarly, we can also prove that for $0\le\nu\le8\left(1+4\sigma^2\right)\gamma$,
\begin{equation}
\begin{aligned}
\mathbb{P}\left[\Big|\sum_{k,t}\left\{\left(V^{(k)}_{t,ij}\right)^2-2\left(r-\tilde{\rho}^{(k)}_{ij}\right)\right\}\Big|
>\frac{2nK}{\sqrt{s_is_j}}\nu\right]
\le2\exp\left\{-\frac{nK\nu^2}{128\left(1+4\sigma^2\right)^2\gamma^2}\right\}
\end{aligned}
\end{equation}
Thus according to \eqref{eq:Ineqstar1}, we have
\begin{equation}
\begin{aligned}
\text{\eqref{eq:tailprob1}}\le & \mathbb{P}\left[\Big|\sum_{k,t}
\left\{\left(U^{(k)}_{t,ij}\right)^2-2\left(r+\tilde{\rho}^{(k)}_{ij}\right)\right\}\Big|
>\frac{2nK\nu}{\sqrt{s_is_j}}\right] \\ & + \mathbb{P}\left[\Big|\sum_{k,t}
\left\{\left(V^{(k)}_{t,ij}\right)^2-2\left(r-\tilde{\rho}^{(k)}_{ij}\right)\right\}\Big|
>\frac{2nK\nu}{\sqrt{s_is_j}}\right] \\ \le &
4\exp\left\{-\frac{nK\nu^2}{128\left(1+4\sigma^2\right)^2\gamma^2}\right\}
\end{aligned}
\end{equation}
i.e., 
\begin{equation}
\begin{aligned}
\mathbb{P}\left[\Big|\sum_{k=1}^K\frac{1}{K}\left(\hat{\Sigma}^{(k)}_{ij}-\bar{\Sigma}^{(k)}_{ij}\right)\Big|>\nu\right]&=
\mathbb{P}\left[\Big|\frac{1}{nK}\sum_{k=1}^{K}\sum_{t=1}^n\left(X_{t,i}^{(k)}X_{t,j}^{(k)}
-\bar{\Sigma}_{ij}^{(k)}\right)\Big|>\nu\right] \\ &
\le4\exp\left\{-\frac{nK\nu^2}{128\left(1+4\sigma^2\right)^2\gamma^2}\right\}
\end{aligned}
\end{equation}
for $0\le\nu\le8\left(1+4\sigma^2\right)\gamma$.
Then consider the $\ell_{\infty}$-norm of $\hat{\Sigma}^{(k)}-\bar{\Sigma}^{(k)}$. Since $\hat{\Sigma}^{(k)},\ \bar{\Sigma}^{(k)}$ are all symmetric and $N\times N$, we have the following bound:
\begin{equation}
\begin{aligned}
\mathbb{P}\left[\|\sum_{k=1}^K\frac{1}{K}\left(\hat{\Sigma}^{(k)}-\bar{\Sigma}^{(k)}\right)\|_{\infty}>\nu\right]&\le
\frac{N(N+1)}{2}\mathbb{P}\left[\Big|\sum_{k=1}^K\frac{1}{K}\left(\hat{\Sigma}^{(k)}_{ij}-\bar{\Sigma}^{(k)}_{ij}\right)\Big|>\nu\right]
\\ &\le
2N(N+1)\exp\left\{-\frac{nK\nu^2}{128\left(1+4\sigma^2\right)^2\gamma^2}\right\}
\end{aligned}
\end{equation}
for $0\le\nu\le8\left(1+4\sigma^2\right)\gamma$, which completes our proof of Lemma \ref{lem:TailSampleCov_fixed}.
\end{proof}

Now consider the setting when $\{\bar{\Sigma}^{(k)}\}_{k=1}^K$ are randomly generated based on Definition \ref{def:msGGM}. According to Lemma \ref{lem:TailSampleCov_fixed}, we have
\begin{equation} \label{eq:IneqSigmaNorm_cond_ij}
\begin{aligned}
\mathbb{P}\left[\left|\sum_{k=1}^K\frac{1}{K}\left(\hat{\Sigma}^{(k)}_{ij}-\bar{\Sigma}^{(k)}_{ij}\right)\right|>\nu
\Big| \{\bar{\Sigma}^{(k)}\}_{k=1}^K\right]\le
\exp\left\{-\frac{nK\nu^2}{128\left(1+4\sigma^2\right)^2\gamma^2}\right\}
\end{aligned}
\end{equation}
\begin{equation} \label{eq:IneqSigmaNorm_cond}
\begin{aligned}
\mathbb{P}\left[\|\sum_{k=1}^K\frac{1}{K}\left(\hat{\Sigma}^{(k)}-\bar{\Sigma}^{(k)}\right)\|_{\infty}>\nu 
\Big| \{\bar{\Sigma}^{(k)}\}_{k=1}^K\right]\le
2N(N+1)\exp\left\{-\frac{nK\nu^2}{128\left(1+4\sigma^2\right)^2\gamma^2}\right\}
\end{aligned}
\end{equation}
for $\hat{\Sigma}^{(k)}=\frac{1}{n}\sum_{t=1}^nX_t^{(k)}(X_t^{(k)})^{\text{T}}$, $1\le i,j\le N$, and  $0\le\nu\le8\left(1+4\sigma^2\right)\gamma$ with $\gamma$ specified in \eqref{eq:condtion2} of the corrected condition (ii) in Definition \ref{def:msGGM}. 

Then by the law of total expectation (see e.g., \cite{weiss2005course}), we have
\begin{equation*}
\begin{aligned}
\mathbb{P}\left[\left|\sum_{k=1}^K\frac{1}{K}\left(\hat{\Sigma}^{(k)}_{ij}-\bar{\Sigma}^{(k)}_{ij}\right)\right|>\nu\right]=&
\mathbb{E}_{\{\bar{\Sigma}^{(k)}\}_{k=1}^K}\left[\mathbb{P}\left[\left|\sum_{k=1}^K\frac{1}{K}\left(\hat{\Sigma}^{(k)}_{ij}-
\bar{\Sigma}^{(k)}_{ij}\right)\right|>\nu \Big| \{\bar{\Sigma}^{(k)}\}_{k=1}^K\right]\right]\\ \le&
\mathbb{E}_{\{\bar{\Sigma}^{(k)}\}_{k=1}^K}\left[\exp\left\{-\frac{nK\nu^2}{128\left(1+4\sigma^2\right)^2\gamma^2}\right\}\right]
\\ =& \exp\left\{-\frac{nK\nu^2}{128\left(1+4\sigma^2\right)^2\gamma^2}\right\}
\end{aligned}
\end{equation*}
Therefore,
\begin{equation*}
\begin{aligned}
\mathbb{P}\left[\|\sum_{k=1}^K\frac{1}{K}\left(\hat{\Sigma}^{(k)}-\bar{\Sigma}^{(k)}\right)\|_{\infty}>\nu \right]=&
\mathbb{E}_{\{\bar{\Sigma}^{(k)}\}_{k=1}^K}\left[
\mathbb{P}\left[\|\sum_{k=1}^K\frac{1}{K}\left(\hat{\Sigma}^{(k)}-\bar{\Sigma}^{(k)}\right)\|_{\infty}>\nu 
\Big| \{\bar{\Sigma}^{(k)}\}_{k=1}^K\right]\right] \\ \le &
\mathbb{E}_{\{\bar{\Sigma}^{(k)}\}_{k=1}^K}\left[
2N(N+1)\exp\left\{-\frac{nK\nu^2}{128\left(1+4\sigma^2\right)^2\gamma^2}\right\}\right] \\ =&
2N(N+1)\exp\left\{-\frac{nK\nu^2}{128\left(1+4\sigma^2\right)^2\gamma^2}\right\}
\end{aligned}
\end{equation*}
which completes the proof of Lemma \ref{lem:TailSampleCov}. Also notice that the proof above does not rely on any assumption on the distribution of $\{\bar{\Sigma}^{(k)}\}_{k=1}^K$. Thus, \eqref{eq:IneqSigmaNorm_ij} and \eqref{eq:IneqSigmaNorm} hold as long as condition (iii), (iv) and (v) in Definition \ref{def:msGGM} are satisfied.

\section{Proof of Lemma \ref{lem:Taildelta}} \label{sec:taildelta}
By definition, $H$ is a function that maps $K$ matrices to a symmetric matrix of dimension $N$, since $\bar{\Omega}^{(k)}=\bar{\Omega}+\Delta^{(k)}\succ0$ with probability 1 according to condition (ii) in Definition \ref{def:msGGM}.
For $\forall k\in\{1,\dots,K\}$, let $\{\Delta^{(1)},\dots,\Delta^{(k)},\dots, \Delta^{(K)},\Delta^{\prime(k)}\}$ be an i.i.d. family of random matrices following distribution $P$ in Definition \ref{def:msGGM}. Consider $H_{1}^{(k)}=H(\Delta^{(1)},\dots,\Delta^{(k)},\dots, \Delta^{(K)})$ and $H_{2}^{(k)}=H(\Delta^{(1)},\dots,\Delta^{\prime(k)},\dots, \Delta^{(K)})$. We have
\begin{equation} \label{eq:dH}
\begin{aligned}
\vertiii{H_1^{(k)}-H_2^{(k)}}_2=&\vertiii{\frac{1}{K}(\bar{\Omega}+\Delta^{\prime(k)})^{-1}-(\bar{\Omega}+\Delta^{(k)})^{-1}}_2\\=&
\frac{1}{K}\vertiii{(\bar{\Omega}+\Delta^{\prime(k)})^{-1}-\bar{\Omega}^{-1}+\bar{\Omega}^{-1}-(\bar{\Omega}+\Delta^{(k)})^{-1}}_2\\ \le&
\frac{1}{K}\vertiii{(\bar{\Omega}+\Delta^{\prime(k)})^{-1}-\bar{\Omega}^{-1}}_2+\frac{1}{K}
\vertiii{(\bar{\Omega}+\Delta^{(k)})^{-1}-\bar{\Omega}^{-1}}_2
\end{aligned}
\end{equation}
Since $\mathbb{P}_{\Delta\sim P}[\vertiii{\Delta}_2\le c_{\max}\le \frac{\lambda_{\min}}{2}]=1$ with $\lambda_{\min}=\lambda_{\min}(\bar{\Omega})$ by \eqref{eq:condtion2} and since $\bar{\Omega}\succ0$, we have
\begin{align*}
\vertiii{(\bar{\Omega}+\Delta^{(k)})^{-1}-\bar{\Omega}^{-1}}_2\le \frac{c_{\max}}{\lambda_{\min}(\lambda_{\min}-c_{\max})}
\le \frac{2c_{\max}}{\lambda_{\min}^2}
\end{align*}
and
\begin{align*}
\vertiii{(\bar{\Omega}+\Delta^{\prime(k)})^{-1}-\bar{\Omega}^{-1}}_2\le \frac{c_{\max}}{\lambda_{\min}(\lambda_{\min}-c_{\max})}
\le \frac{2c_{\max}}{\lambda_{\min}^2}
\end{align*}
according to Equation (7.2) in \cite{ELGHAOUI2002171}. Plug the above inequalities in \eqref{eq:dH} and we can get
\begin{equation} \label{eq:dH1}
\begin{aligned}
\vertiii{H_1^{(k)}-H_2^{(k)}}_2\le
\frac{1}{K}\vertiii{(\bar{\Omega}+\Delta^{\prime(k)})^{-1}-\bar{\Omega}^{-1}}_2+\frac{1}{K}
\vertiii{(\bar{\Omega}+\Delta^{(k)})^{-1}-\bar{\Omega}^{-1}}_2\le \frac{4c_{\max}}{K\lambda_{\min}^2}
\end{aligned}
\end{equation}

For $k=1,\dots,K$, define $A_k=\frac{4c_{\max}}{K\lambda_{\min}^2}I_N$ with $I_N\in\mathbb{R}^{N\times N}$ being an identity matrix. Then by \eqref{eq:dH1}, we have
\begin{equation*}
\begin{aligned}
(H_1^{(k)}-H_2^{(k)})^2\preceq A_k^2
\end{aligned}
\end{equation*}
where $X\preceq Y\iff Y-X\succeq0$.

Define $\sigma_{\Delta}^2:=\vertiii{\sum_{k=1}^KA_k^2}_2=\sum_{k=1}^K \left(\frac{4c_{\max}}{K\lambda_{\min}^2}\right)^2=\frac{16c_{\max}^2}{K\lambda_{\min}^4}$. Then according to Corollary 7.5 in \cite{Tropp2011}, we have
\begin{equation} \label{eq:Ineqdelta0}
\begin{aligned}
\mathbb{P}\left[\lambda_{\max}(H-\mathbb{E}[H])>t\right]\le
N\exp\left\{-\frac{t^2}{8\sigma_{\Delta}^2}\right\}\le
N\exp\left\{-\frac{\lambda_{\min}^4Kt^2}{128c_{\max}^2}\right\}
\end{aligned}
\end{equation}

Consider $-H(\Delta^{(1)},\dots,\Delta^{(K)})$. We have
$$((-H_1^{(k)})-(-H_2^{(k)}))^2\preceq A_k^2$$
The conditions of Corollary 7.5 in \cite{Tropp2011} are also satisfied. Thus, we have
\begin{equation} \label{eq:Ineqdelta01}
\begin{aligned}
\mathbb{P}\left[-\lambda_{\min}(H-\mathbb{E}[H])>t\right]=
\mathbb{P}\left[\lambda_{\max}((-H)-(-\mathbb{E}[H]))>t\right]\le
N\exp\left\{-\frac{t^2}{8\sigma_{\Delta}^2}\right\}\le
N\exp\left\{-\frac{\lambda_{\min}^4Kt^2}{128c_{\max}^2}\right\}
\end{aligned}
\end{equation}
By \eqref{eq:Ineqdelta0} and \eqref{eq:Ineqdelta01}, we have
\begin{equation}
\begin{aligned}
\mathbb{P}\left[\vertiii{H-\mathbb{E}[H]}_2>t\right]=&
\mathbb{P}\left[\lambda_{\max}(H-\mathbb{E}[H])>t,-\lambda_{\min}(H-\mathbb{E}[H])>t\right]\\
\le & \mathbb{P}\left[\lambda_{\max}(H-\mathbb{E}[H])>t\right]+\mathbb{P}\left[-\lambda_{\min}(H-\mathbb{E}[H])>t\right]\\ \le&
2N\exp\left\{-\frac{\lambda_{\min}^4Kt^2}{128c_{\max}^2}\right\}
\end{aligned}
\end{equation}
which gives us \eqref{eq:Ineqdelta}.

\section{Proof of Lemma \ref{lem:suffconsis}} \label{sec:suffconsis}
For $\xi\in(0,\delta^*]$, we have proved that if $\|\sum_{k=1}^K\frac{T^{(k)}}{T}W^{(k)}\|_{\infty}\le \xi$
then $\|\Delta\|_{\infty}\le2\kappa_{\bar{\Gamma}}\left(\frac{8}{\alpha}+1\right)\xi$, $\tilde{\Omega}=\hat{\Omega}$ 
and $\text{supp}(\hat{\Omega})\subseteq\text{supp}(\bar{\Omega})$.

Therefore if we further assume that
$$
\frac{\omega_{\min}}{2}\ge2\kappa_{\bar{\Gamma}}\left(\frac{8}{\alpha}+1\right)\xi
$$
we will have
$$
\frac{\omega_{\min}}{2}\ge\|\Delta\|_{\infty}=\|\hat{\Omega}-\bar{\Omega}\|_{\infty}
$$
Then for any $(i,j)\in S^c=\left[\text{supp}(\bar{\Omega})\right]^c$, $\bar{\Omega}_{ij}=0$, we have $\left[\text{supp}(\bar{\Omega})\right]^c\subseteq\left[\text{supp}(\hat{\Omega})\right]^c$ and thus $(i,j)\in\left[\text{supp}(\hat{\Omega})\right]^c$, $\hat{\Omega}_{ij}=0=\bar{\Omega}_{ij}$

For any $(i,j)\in S=\text{supp}(\bar{\Omega})$, we have
$$
|\hat{\Omega}_{ij}-\bar{\Omega}_{ij}|\le\|\hat{\Omega}-\bar{\Omega}\|_{\infty}\le\frac{\omega_{\min}}{2}=
\frac{1}{2}\min_{1\le k,l\le N}\bar{\Omega}_{kl}\le\frac{1}{2}|\bar{\Omega}_{ij}|
$$
$$
\Rightarrow
-\frac{1}{2}|\bar{\Omega}_{ij}|\le \hat{\Omega}_{ij}-\bar{\Omega}_{ij}\le \frac{1}{2}|\bar{\Omega}_{ij}|
$$
If $\bar{\Omega}_{ij}>0$, then
$$
-\frac{1}{2}\bar{\Omega}_{ij}\le \hat{\Omega}_{ij}-\bar{\Omega}_{ij}
$$
$$
\hat{\Omega}_{ij}\ge \frac{1}{2}\bar{\Omega}_{ij}>0
$$
If $\bar{\Omega}_{ij}<0$, then
$$
\hat{\Omega}_{ij}-\bar{\Omega}_{ij}\le-\frac{1}{2}\bar{\Omega}_{ij}
$$
$$
\hat{\Omega}_{ij}\le \frac{1}{2}\bar{\Omega}_{ij}<0
$$
In conclusion, $\text{sign}(\hat{\Omega}_{ij})=\text{sign}(\bar{\Omega}_{ij})$ for $\forall\ i,j\in\left\{1,2,...,N\right\}$. The estimate $\hat{\Omega}$ in \eqref{eq:estimator} is sign-consistent.

\section{Proof of Lemma \ref{lem:novel1}} \label{sec:lemnovel1}
Plug $\mu=\bar{\Sigma}^{(K+1),S^c}=(\bar{\Sigma}^{(K+1)}_{S^c},0)$ and $\nu=\text{diag}(\bar{\Sigma}^{(K+1)}-\hat{\Sigma}^{(K+1)})$ in \eqref{eq:lagrangian}. We have the following optimization problem
\begin{align*}
\hat{\Omega}^{(K+1)}=\arg\min_{\Omega\in\mathcal{S}_{++}^{N}}\ell^{(K+1)}(\Omega)+\lambda\|\Omega\|_1+
\langle\bar{\Sigma}^{(K+1),S^c},\Omega\rangle 
+\langle\text{diag}(\bar{\Sigma}^{K+1}-\hat{\Sigma}^{K+1}),\text{diag}(\Omega-\hat{\Omega})\rangle
\end{align*}
Now we can prove with the five steps in the primal-dual witness approach.

\subsection{Step 1}
For $(\Omega_{S^{(K+1)}},0)\in\mathcal{S}_{++}^{N}$, we need to verify $[\nabla^2\ell^{(K+1)}(\Omega)]_{S^{(K+1)}S^{(K+1)}}\succ0$. In fact,
\begin{equation} \label{eq:gradlossnovel}
\nabla\ell^{(K+1)}(\Omega)=\hat{\Sigma}^{(K+1),S}-\Omega^{-1}
\end{equation}
\begin{equation}
\nabla^2\ell^{(K+1)}(\Omega)=\Gamma(\Omega)=\Omega^{-1}\otimes\Omega^{-1}
\end{equation}
For $(\Omega_{S^{(K+1)}},0)\in\mathcal{S}_{++}^{N}$, we have $\Gamma((\Omega_{S^{(K+1)}},0))\succ0$, $\nabla^2\ell^{(K+1)}(\Omega)\succ0$. Thus following the same steps in section \ref{subsection:step1} , we can prove $[\nabla^2\ell^{(K+1)}(\Omega)]_{S^{(K+1)}S^{(K+1)}}\succ0$.

\subsection{Step 2}
Construct the primal variable $\tilde{\Omega}$ by making $\tilde{\Omega}_{[S^{(K+1)}]^c}=0$ and solving the restricted problem:
\begin{equation} \label{eq:novelstep2}
\begin{aligned}
\tilde{\Omega}_{S^{(K+1)}}=&\arg\min_{\left(\Omega_{S^{(K+1)}},0\right)\in\mathcal{S}_{++}^N}
\ell^{(K+1)}\left(\left(\Omega_{S^{(K+1)}},0\right)\right)+\lambda\|\Omega_{S^{(K+1)}}\|_1
\\&+\langle\bar{\Sigma}^{(K+1),S^c},\left(\Omega_{S^{(K+1)}},0\right)\rangle +\langle\text{diag}(\bar{\Sigma}^{K+1}-\hat{\Sigma}^{K+1}),\text{diag}(\left(\Omega_{S^{(K+1)}},0\right)-\hat{\Omega})\rangle
\end{aligned}
\end{equation}

\subsection{Step 3}
Choose the dual variable $\tilde{Z}$ in order to fulfill the complementary slackness condition of \eqref{eq:lagrangian}:
\begin{equation} \label{eq:novelstep3}
\left\{
\begin{aligned}
&\tilde{Z}_{ij}=1,\ \ \text{if } \tilde{\Omega}_{ij}>0\\
&\tilde{Z}_{ij}=-1,\ \ \text{if } \tilde{\Omega}_{ij}<0\\
&\tilde{Z}_{ij}\in[-1,1],\ \ \text{if } \tilde{\Omega}_{ij}=0
\end{aligned}
\right.
\end{equation}
Therefore we have
\begin{equation} \label{eq:step3ineq_novel}
\begin{aligned}
\|\tilde{Z}\|_{\infty}\le 1
\end{aligned}
\end{equation}

\subsection{Step 4}
$\tilde{Z}$ is the subgradient of $\|\tilde{\Omega}\|_1$. Solve for the dual variable $\tilde{Z}_{[S^{(K+1)}]^c}$ in order that $(\tilde{\Omega},\tilde{Z})$ fulfills the stationarity condition of \eqref{eq:lagrangian}:
\begin{equation}
\left[\nabla\ell^{(K+1)}\left(\left(\tilde{\Omega}_{S^{(K+1)}},0\right)\right)\right]_{S^{(K+1)}}
+\lambda\tilde{Z}_{S^{(K+1)}} + I_N\text{diag}(\bar{\Sigma}^{(K+1)}-\hat{\Sigma}^{(K+1)})
=0
\end{equation}
\begin{equation} \label{eq:novelstation2}
\left[\nabla\ell^{(K+1)}\left(\left(\tilde{\Omega}_{S^{(K+1)}},0\right)\right)\right]_{[S^{(K+1)}]^c}
+\lambda\tilde{Z}_{[S^{(K+1)}]^c}+\bar{\Sigma}^{(K+1),S^c}_{[S^{(K+1)}]^c}=0
\end{equation}
where $I_N\in\mathbb{R}^{N\times N}$ is an identity matrix.

\subsection{Step 5}
Now we need to verify that the dual variable solved by Step 4 satisfied the strict dual feasibility condition:
\begin{equation}
\|\tilde{Z}_{[S^{(K+1)}]^c}\|_{\infty}<1
\end{equation}
If we can show the strict dual feasibility condition holds, we can claim that the solution in \eqref{eq:novelstep2} is equal to the solution in \eqref{eq:estimatorNovel1}, i.e., $\tilde{\Omega}=\hat{\Omega}^{(K+1)}$. Thus we will have
\begin{equation*}
\text{supp}\left(\hat{\Omega}^{(K+1)}\right)=\text{supp}\left(\tilde{\Omega}\right)\subseteq
S^{(K+1)}=\text{supp}\left(\bar{\Omega}^{(K+1)}\right)
\end{equation*}

\subsection{Proof of the strict dual feasibility condition}
Plug \eqref{eq:gradlossnovel} in the stationarity condition of \eqref{eq:estimatorNovel1}, we have
\begin{equation} \label{eq:novelstation.1}
\hat{\Sigma}^{(K+1),S}-\tilde{\Omega}^{-1}+\lambda\tilde{Z}+\bar{\Sigma}^{(K+1),S^c}
+I_N\text{diag}(\bar{\Sigma}^{K+1}-\hat{\Sigma}^{K+1})=0
\end{equation}
Define $\Psi:=\tilde{\Omega}-\bar{\Omega}^{(K+1)}$, 
$R(\Psi):=\tilde{\Omega}^{-1}-\bar{\Sigma}^{(K+1)}+\bar{\Sigma}^{(K+1)}\Psi\bar{\Sigma}^{(K+1)}$. Notice that $W^{(K+1)}=\bar{\Sigma}^{(K+1),S_{\text{off}}}-\hat{\Sigma}^{(K+1),S_{\text{off}}}$.
Then we can rewrite \eqref{eq:novelstation.1} as
\begin{equation} \label{eq:novelstation.2}
\begin{aligned}
0=&\hat{\Sigma}^{(K+1),S}-\tilde{\Omega}^{-1}+\lambda\tilde{Z}+\bar{\Sigma}^{(K+1),S^c}
+I_N\text{diag}(\bar{\Sigma}^{K+1}-\hat{\Sigma}^{K+1})\\=
&\hat{\Sigma}^{(K+1),S}-(\tilde{\Omega}-\bar{\Sigma}^{(K+1)}+\bar{\Sigma}^{(K+1)}\Psi\bar{\Sigma}^{(K+1)})
-\bar{\Sigma}^{(K+1)}+\bar{\Sigma}^{(K+1)}\Psi\bar{\Sigma}^{(K+1)}+\bar{\Sigma}^{(K+1),S^c}
\\&+I_N\text{diag}(\bar{\Sigma}^{K+1}-\hat{\Sigma}^{K+1}) + \lambda\tilde{Z}
\\ =&\hat{\Sigma}^{(K+1),S_{\text{off}}}+I_N\text{diag}(\hat{\Sigma}^{(K+1)})-R(\Psi)-\bar{\Sigma}^{(K+1),S}+
I_N\text{diag}(\bar{\Sigma}^{K+1}-\hat{\Sigma}^{K+1})+ \lambda\tilde{Z}
\\=&\hat{\Sigma}^{(K+1),S_{\text{off}}}-\bar{\Sigma}^{(K+1),S_{\text{off}}}
+\bar{\Sigma}^{(K+1)}\Psi\bar{\Sigma}^{(K+1)}-R(\Psi)+\lambda\tilde{Z}
\\=&W^{(K+1)}+\bar{\Sigma}^{(K+1)}\Psi\bar{\Sigma}^{(K+1)}-R(\Psi)+\lambda\tilde{Z}
\end{aligned}
\end{equation}
Now apply Lemma \ref{lem:strictdual} with $K=1$ and we can get Lemma \ref{lem:novel1}.

\section{Proof of Lemma \ref{lem:suffconsis_novel}} \label{sec:suffconsis_novel}
For $\xi\in(0,\delta^{(K+1),*}]$, in Lemma \ref{lem:novel1}, we have proved that if $\|W^{(K+1)}\|_{\infty}\le \xi$
then
$\|\hat{\Omega}^{(K+1)}-\bar{\Omega}^{(K+1)}\|_{\infty}\le2\kappa_{\bar{\Gamma}^{(K+1)}}\left(\frac{8}{\alpha^{(K+1)}}+1\right)\xi$
and $\text{supp}(\hat{\Omega}^{(K+1)})\subseteq\text{supp}(\bar{\Omega}^{(K+1)})$.

Therefore if we further assume that
$$
\frac{\omega_{\min}^{(K+1)}}{2}\ge2\kappa_{\bar{\Gamma}^{(K+1)}}\left(\frac{8}{\alpha^{(K+1)}}+1\right)\xi
$$
we will have
$$
\frac{\omega_{\min}^{(K+1)}}{2}\ge\|\hat{\Omega}^{(K+1)}-\bar{\Omega}^{(K+1)}\|_{\infty}
$$
Then for any $(i,j)\in [S^{(K+1)}]^c=\left[\text{supp}(\bar{\Omega}^{(K+1)})\right]^c$, $\bar{\Omega}^{(K+1)}_{ij}=0$, we have $\left[\text{supp}(\bar{\Omega}^{(K+1)})\right]^c\subseteq\left[\text{supp}(\hat{\Omega}^{(K+1)})\right]^c$ and thus $(i,j)\in\left[\text{supp}(\hat{\Omega}^{(K+1)})\right]^c$, $\hat{\Omega}^{(K+1)}_{ij}=0=\bar{\Omega}^{(K+1)}_{ij}$

For any $(i,j)\in S^{(K+1)}=\text{supp}(\bar{\Omega}^{(K+1)})$, we have
$$
|\hat{\Omega}^{(K+1)}_{ij}-\bar{\Omega}^{(K+1)}_{ij}|\le\|\hat{\Omega}^{(K+1)}-\bar{\Omega}^{(K+1)}\|_{\infty}
\le\frac{\omega_{\min}^{(K+1)}}{2}=
\frac{1}{2}\min_{1\le k,l\le N}\bar{\Omega}^{(K+1)}_{kl}\le\frac{1}{2}|\bar{\Omega}^{(K+1)}_{ij}|
$$
$$
\Rightarrow
-\frac{1}{2}|\bar{\Omega}^{(K+1)}_{ij}|\le \hat{\Omega}^{(K+1)}_{ij}-\bar{\Omega}^{(K+1)}_{ij}\le \frac{1}{2}|\bar{\Omega}^{(K+1)}_{ij}|
$$
If $\bar{\Omega}^{(K+1)}_{ij}>0$, then
$$
-\frac{1}{2}\bar{\Omega}^{(K+1)}_{ij}\le \hat{\Omega}^{(K+1)}_{ij}-\bar{\Omega}^{(K+1)}_{ij}
$$
$$
\hat{\Omega}^{(K+1)}_{ij}\ge \frac{1}{2}\bar{\Omega}^{(K+1)}_{ij}>0
$$
If $\bar{\Omega}^{(K+1)}_{ij}<0$, then
$$
\hat{\Omega}^{(K+1)}_{ij}-\bar{\Omega}^{(K+1)}_{ij}\le-\frac{1}{2}\bar{\Omega}^{(K+1)}_{ij}
$$
$$
\hat{\Omega}^{(K+1)}_{ij}\le \frac{1}{2}\bar{\Omega}^{(K+1)}_{ij}<0
$$
In conclusion, $\text{sign}(\hat{\Omega}^{(K+1)}_{ij})=\text{sign}(\bar{\Omega}^{(K+1)}_{ij})$ for $\forall\ i,j\in\left\{1,2,...,N\right\}$. The estimate $\hat{\Omega}^{(K+1)}$ in \eqref{eq:estimatorNovel1} is sign-consistent.

\end{appendices}

\end{document}